\documentclass[twoside,11pt]{article}

\RequirePackage[OT1]{fontenc}
\RequirePackage{amsmath,amssymb,enumerate,lmodern}
\RequirePackage[numbers]{natbib}
\usepackage[colorlinks=false,allbordercolors={1 1 1}]{hyperref}
\usepackage{jmlr2e}

\newtheorem{assump}{Assumption}

\usepackage{mathtools}
\usepackage{amsmath}

\font\tencmmib=cmmib10 \skewchar\tencmmib '60
\newfam\cmmibfam
\textfont\cmmibfam=\tencmmib

\def\bbox{\quad\hbox{\vrule \vbox{\hrule \vskip2pt \hbox{\hskip2pt
\vbox{\hsize=1pt}\hskip2pt} \vskip2pt\hrule}\vrule}}
\def\lessim{\ \lower4pt\hbox{$
\buildrel{\displaystyle <}\over\sim$}\ }
\def\gessim{\ \lower4pt\hbox{$\buildrel{\displaystyle >}
\over\sim$}\ }

\def\eps{\varepsilon}

\def\go0{\to 0}

\def\leftitem#1{\item{\hbox to\parindent{\enspace#1\hfill}}}

\def\qed{{$\hfill \bbox$}}

\def\sg{\sigma}

\def\sg2{\sigma^2}

\def\__{_{\infty}}

\jmlrheading{16}{2015}{..}{7/15}{9/15}{Vladimir Koltchinskii and Dong Xia}
\ShortHeadings{Low Rank Density Matrices Estimation}{Koltchinskii and Xia}
\begin{document}

\title{Optimal Estimation of Low Rank Density Matrices}

\author{Vladimir Koltchinskii\thanks{Supported in part by NSF Grants DMS-1509739, DMS-1207808, CCF-1523768 and CCF-1415498} \email vlad@math.gatech.edu\\
 \bf{Dong Xia}  \thanks{Supported in part by NSF Grant DMS-1207808} \email dxia7@math.gatech.edu\\
\addr School of Mathematics\\
Georgia Institute of Technology\\
Atlanta, GA 30332, USA.
}


\editor{Alex Gammerman and Vladimir Vovk}

\maketitle

\begin{abstract}%
The density matrices are positively semi-definite Hermitian matrices of unit trace that describe the state of a quantum system. The goal of the paper is to develop minimax lower bounds on error rates of estimation of low rank density matrices in trace regression models 
used in quantum state tomography (in particular, in the case of Pauli measurements)
with explicit dependence of the bounds on the rank and other complexity parameters.  
Such bounds are established for several statistically relevant distances, including quantum versions of Kullback-Leibler divergence (relative entropy distance) and of Hellinger distance (so called 
Bures distance), and Schatten $p$-norm distances. Sharp upper bounds and oracle inequalities 
for least squares estimator with von Neumann entropy penalization are obtained 
showing that minimax lower bounds are attained (up to logarithmic factors) for these distances. 
\end{abstract}


\section{Introduction}
\label{intro}

{\it This paper deals with optimality properties of estimators of density matrices, describing states of quantum systems, that are based on penalized empirical risk minimization with specially designed complexity penalties such as von Neumann entropy of the state.
Alexey Chervonenkis was a co-founder of the theory of empirical risk minimization that is of crucial importance in 
machine learning, but he also had very broad interests that included, in particular, quantum mechanics. By the choice 
of the topic, we would like to honor the memory of this great man and great scientist.}

Let $\mathbb{M}_m(\mathbb{C})$ be the set of all $m\times m$ matrices with complex entries and 
let ${\mathbb H}_m=\mathbb{H}_m(\mathbb{C})\subset \mathbb{M}_m(\mathbb{C})$ be the set of all Hermitian matrices: ${\mathbb H}_m=\{A\in {\mathbb M}_m({\mathbb C}):A=A^{\ast}\},$ $A^{\ast}$ denoting the adjoint matrix of $A.$
For $A\in\mathbb{H}_m,$ $\text{tr}(A)$ denotes the trace of $A$ and $A\succcurlyeq 0$ means that $A$ is positively semi-definite. Let $\mathcal{S}_m:=\left\{S\in\mathbb{H}_m: 
S\succcurlyeq 0, \textrm{tr}(S)=1\right\}$ be the set of all positively semi-definite Hermitian matrices of unit trace called {\it density matrices}.  
In quantum mechanics, the state of a quantum system is usually characterized by a density matrix $\rho\in\mathcal{S}_m$ (or, more generally, by a self-adjoint positively semi-definite operator of unit trace acting 
in an infinite-dimensional Hilbert space, called a density operator). Often, very large density 
matrices are needed to represent or to approximate the density operator of the state. For instance, 
for a quantum system consisting of $b$ qubits, the density matrices are of the size $m\times m$
with $m=2^b,$ so the dimension of the density matrix grows exponentially with $b.$  
For instance, for a $10$ qubit system, one has to deal with matrices that have $2^{20}$
entries. Thus, it becomes natural in the problems of statistical estimation of density matrix 
$\rho$ to take an advantage of the fact that it might be low rank, or nearly low rank 
(that is, it could be well approximated by low rank matrices) which reduces the complexity
of the estimation problem.

In {\it quantum state tomography (QST)}, the goal is to estimate an unknown state $\rho\in {\mathcal S}_m$ based on a number of specially designed measurements for the system prepared in state $\rho$ (see \citealt{gross2010quantum}, \citealt{gross2011recovering}, \citealt{koltchinskii2011neumann}, \citealt{caioptimal} and references therein).  
Given an observable $A\in {\mathbb H}_m$ with spectral representation $A=\sum_{j=1}^{m'}\lambda_j P_j,$ where $m'\leq m,$ $\lambda_j$ being the eigenvalues of $A$ and $P_j$
being the corresponding mutually orthogonal eigenprojectors, the outcome of a measurement 
of $A$ for the system prepared in state $\rho$ is a random variable $Y$ taking values $\lambda_j$
with probabilities $\textrm{tr}(\rho P_j).$ The expectation of $Y$ is then 
${\mathbb E}_{\rho}Y=\textrm{tr}(\rho A),$ so, $Y$ could be viewed as a noisy 
observation of the value of linear functional $\textrm{tr}(\rho A)$ of the unknown 
density matrix $\rho.$ A common approach is to choose an observable $A$ at random,
assuming that it is the value of a random variable $X$ with some design distribution 
$\Pi$ in the space ${\mathbb H}_m.$ More precisely, given a sample of $n$ i.i.d. copies 
$X_1,\dots, X_n$ of $X,$ $n$ measurements are being performed for the system identically
prepared $n$ times in state $\rho$ resulting in outcomes $Y_1,\dots, Y_n.$ Based on 
the data $(X_1,Y_1), \dots, (X_n,Y_n),$ the goal is to estimate the target density matrix 
$\rho.$ Clearly, the observations satisfy the following model 
\begin{equation}
\label{trace_regression}
Y_j={\rm tr}(\rho X_j)+ \xi_j,\ j=1,\dots, n, 
\end{equation}
where $\{\xi_j\}$ is a random noise consisting of $n$ i.i.d. random variables satisfying 
the condition ${\mathbb E}_{\rho}(\xi_j|X_j)=0, j=1,\dots, n.$ This is a special 
case of so called {\it trace regression model} intensively studied in the recent literature
(see, e.g., \citealt{koltchinskii2011nuclear}, \citealt{koltchinskii2011oracle} and references therein). 

\subsection{Assumptions}

A common choice of design distribution in this type of problems is so called {\it uniform sampling 
from an orthonormal basis} described in the following assumptions.

\begin{assump}
\label{orthonorm}
Let ${\mathcal E}=\{E_1,\dots, E_{m^2}\}\subset {\mathbb H}_m$ be an orthonormal basis of ${\mathbb H}_m$ with respect to the Hilbert--Schmidt inner 
product: $\langle A,B\rangle=\textrm{tr}(AB).$ 
Moreover, suppose that, for some $U>0,$ 
$$\|E_j\|_{\infty}\leq U, j=1,\dots, n,$$ 
where $\|\cdot\|_{\infty}$
denotes the operator norm (the spectral norm). 
\end{assump}

Since $\|E_j\|_2=1,$ where $\|\cdot\|_2$ denotes 
the Hilbert--Schmidt (or Frobenius) norm, we can assume that $U\leq 1.$ Moreover,
$U\geq m^{-1/2}$ since $1=\|E_j\|_2\leq m^{1/2}\|E_j\|_{\infty}\leq m^{1/2}U.$

\begin{assump}
\label{unif_sample}
Let $\Pi$ be the uniform distribution in the finite set ${\mathcal E}$ (see Assumption \ref{orthonorm}),
let $X$ be a random variable sampled from $\Pi$ and let $X_1,\dots , X_n$ be i.i.d. copies of $X.$
\end{assump}

It will be assumed in what follows that assumptions \ref{orthonorm} and \ref{unif_sample} hold (unless it is stated otherwise). Under these assumptions, $Y_1,\dots, Y_n$ could be viewed 
as noisy observations of a random sample of Fourier coefficients $\langle \rho, X_1\rangle,
\dots, \langle \rho, X_n\rangle$ of the target density matrix $\rho$ in the basis ${\mathcal E}.$
The above model (in which $X_1,\dots, X_n$ are uniformly sampled from an orthonormal basis 
and $Y_1,\dots, Y_n$ are the outcomes of measurements of the observables $X_1,\dots, X_n$ for the system being 
identically prepared $n$ times in the same state $\rho$) will be called in what follows the {\it standard QST model}.
It is a special case of {\it trace regression model with bounded response}:

\begin{assump}[Trace regression with bounded responce]
\label{bounded_response}
Suppose that Assumption \ref{orthonorm} holds and let $(X,Y)$ be a random couple such 
that $X$ is sampled from the uniform distribution $\Pi$ in an orthonormal basis ${\mathcal E}\subset {\mathbb H}_m.$
Suppose also that, for some $\rho\in {\mathcal S}_m,$ ${\mathbb E}(Y|X)=\langle \rho,X\rangle$ a.s. and, for 
some $\bar U>0,$  $|Y|\leq \bar U$ a.s.. The data $(X_1,Y_1),\dots (X_n,Y_n)$ consists of $n$ i.i.d.
copies of $(X,Y).$
\end{assump}

We are also interested in the {\it trace regression model with Gaussian noise}: 

\begin{assump}[Trace regression with Gaussian noise]
\label{Gaussian_noise}
Suppose Assumption \ref{orthonorm} holds and let $(X,Y)$ be a random couple such 
that $X$ is sampled from the uniform distribution $\Pi$ in an orthonormal basis ${\mathcal E}\subset {\mathbb H}_m$
and, for some $\rho\in {\mathcal S}_m,$ 
$
Y=\langle \rho,X\rangle +\xi,
$
where $\xi$ is a normal random variable with mean $0$ and variance $\sigma_{\xi}^2,$
$\xi$ and $X$ being independent.
The data $(X_1,Y_1),\dots (X_n,Y_n)$ consists of $n$ i.i.d.
copies of $(X,Y).$
\end{assump}

Note that this model is not directly applicable to the ``standard QST problem" described above,
where the response variable $Y$ is discrete. However, if the measurements are repeated multiple 
times for each observable $X_j$ and the resulting outcomes are averaged to reduce the variance,
the noise of such averaged measurements becomes approximately Gaussian and it is of interest 
to characterize the estimation error in terms of the variance of the noise.

An important example of an orthonormal basis used in quantum state tomography is 
so called {\it Pauli basis}, see, e.g., \cite{gross2010quantum}, \cite{gross2011recovering}.
The Pauli basis in the space ${\mathbb H}_2$ of $2\times 2$ Hermitian matrices 
(observables in a single qubit system) consists of four matrices $W_1,W_2,W_3,W_4$ 
defined as $W_i=\frac{1}{\sqrt{2}}\sigma_i,\ i=1,\dots, 4,$ where
 \begin{equation*}
  \sigma_1:=\left(\begin{array}{cc}1&0\\0&1 \end{array}\right),\quad \sigma_2:=\left(\begin{array}{cc}0&-i\\i&0 \end{array}\right),
  \quad \sigma_3:=\left(\begin{array}{cc}0&1\\1&0 \end{array}\right),\quad \sigma_4:=\left(\begin{array}{cc}1&0\\0&-1 \end{array}\right).
 \end{equation*}
 It is easy to check that $\{W_0,W_1,W_2,W_3\}$ indeed forms an orthonormal basis in $\mathbb{H}_2.$ The Pauli basis in the space ${\mathbb H}_m$ for $m=2^b$ (the space 
 of observables for a $b$ qubits system) is defined by tensorisation, namely, it consists 
 of $4^b$ tensor products $W_{i_1}\otimes\ldots\otimes W_{i_b}, (i_1,\ldots,i_b)
\in \left\{1,2,3,4\right\}^b$. Let us write these matrices as $E_1,\dots, E_{m^2}$
with $E_1=W_1\otimes\ldots\otimes W_1.$ It is easy to see that each of them has 
eigenvalues $\pm\frac{1}{\sqrt{m}}$ and $\|E_j\|_{\infty}=m^{-1/2},$ so, for this basis,
$U=m^{-1/2}.$ The fact that, for the Pauli basis, the operator norms of basis matrices are as small as possible plays an important role in quantum state tomography 
\citep{gross2010quantum,gross2011recovering,liu2011universal}.
Let $E_j=\frac{1}{\sqrt{m}}Q_j^+-\frac{1}{\sqrt{m}}Q_j^-$ be the spectral 
representation of $E_j.$ Then, an outcome of a measurement of $E_j$ in state 
$\rho$ is a random variable $\tau_j$ taking values $\pm\frac{1}{\sqrt{m}}$
with probabilities $\big<\rho,Q_t^{\pm}\big>.$ Its expectation is ${\mathbb E}_{\rho}\tau_j=\langle \rho,E_j\rangle.$ Of course, there exists a unique representation of density matrix $\rho$ in the Pauli 
basis that can be written as follows:
$\rho=\sum_{j=1}^{m^2}\frac{\alpha_j}{\sqrt{m}}E_j$ with $\alpha_1=1.$
Then, we clearly have ${\mathbb E}_{\rho}\tau_j=\frac{\alpha_j}{\sqrt{m}}$
and $\mathbb{P}_{\rho}\Bigl\{\tau_j=\pm\frac{1}{\sqrt{m}}\Bigr\}=\frac{1\pm \alpha_j}{2}$
(for $j=1,$ this gives $\mathbb{P}_{\rho}\Bigl\{\tau_1=\frac{1}{\sqrt{m}}\Bigr\}=1$).
As a consequence, $\text{Var}_{\rho}(\tau_j)=\frac{1-\alpha_j^2}{m}.$
Note that $\sum_{j=1}^{m^2}\frac{\alpha_j^2}{m}=\|\rho\|_2^2\leq {\rm tr}^2(\rho)=1.$
This implies that there exists $j$ such that $\alpha_j^2 \leq \frac{1}{2}$
and $\text{Var}_{\rho}(\tau_j)\geq \frac{1}{2m}.$ In fact, the number of 
such $j$ must be large, say, at least $\frac{m^2}{2}$ (provided that $m>4$). 
Thus, for ``most" of the values of $j,$ $\text{Var}_{\rho}(\tau_j) \asymp\frac{1}{m}.$
A way to reduce the variance is to repeat the measurement of each observable $X_j$
$K$ times (for a system identically prepared in state $\rho$) and to average the outcomes 
of such $K$ measurements. The resulting response variable is $Y_j=\langle \rho, X_j\rangle +
\xi_j,$ where ${\mathbb E}_{\rho}(\xi_j|X_j)=0$ and ${\mathbb E}_{\rho}(\xi_j^2|X_j)=\text{Var}_{\rho}(Y_j|X_j)=\frac{1-\alpha_{\nu_j}^2}{Km},$ $\nu_j$ being defined by the relationship
$X_j=E_{\nu_j}.$

\subsection{Preliminaries and Notations} 

Some notations will be used throughout the paper. The Euclidean norm in $\mathbb{C}^{m}$
will be denoted by $\|\cdot\|$ and the notation $\langle \cdot,\cdot \rangle$ will be used 
for both the Euclidean inner product in ${\mathbb C}^m$ and for the Hilbert--Schmidt inner 
product in ${\mathbb H}_m.$
$\|\cdot\|_p, p\geq 1$ will be used to denote the 
{\it Schatten p}-norm in $\mathbb{H}_m,$ namely $\|A\|_p^p=
\sum\limits_j^m |\lambda_j(A)|^p,\ A\in\mathbb{H}_m,$ $\lambda_1(A)\geq \ldots \geq \lambda_m(A)$ being the eigenvalues of $A.$ In particular,  $\|\cdot\|_2$ denotes the Hilbert--Schmidt (or Frobenius) norm, $\|\cdot\|_1$ denotes the nuclear (or trace) norm and
$\|\cdot\|_{\infty}$ denotes the operator (or spectral) norm: $\|A\|_{\infty}=\max_{1\leq j\leq m}|\lambda_j(A)|=|\lambda_1(A)|.$ The following well known {\it interpolation inequality} for 
Schatten $p$-norms will be used to extend the bounds proved for some values of $p$
to the whole range of its values. It easily follows from similar bounds for $\ell_p$-spaces. 

\begin{lemma}[Interpolation inequality]
\label{interlem}
 For $1\leq p<q<r\leq\infty$, and let $\mu\in[0,1]$ be such that
\begin{equation*}
 \frac{\mu}{p}+\frac{1-\mu}{r}=\frac{1}{q}.
\end{equation*}
Then, for all $A\in\mathbb{H}_m,$
\begin{equation*}
 \|A\|_q\leq \|A\|_p^{\mu}\|A\|_r^{1-\mu}.
\end{equation*}
\end{lemma}

Given $A\in {\mathbb H}_m,$ define a function $f_A:{\mathbb H}_m\mapsto {\mathbb R}:$
$f_A(x):=\langle A,x\rangle, x\in {\mathbb H}_m.$ For a given random variable $X$ in 
${\mathbb H}_m$ with a distribution $\Pi,$ we have 
$
\|f_A\|_{L_2(\Pi)}^2 = {\mathbb E}f_A^2(X)= {\mathbb E}\langle A,X\rangle^2. 
$
Sometimes, with a minor abuse of notation, we might write $\|A\|_{L_2(\Pi)}^2=
\int_{{\mathbb H}_m}\langle A,x\rangle^2\Pi(dx)=\|f_A\|_{L_2(\Pi)}^2.$ In what 
follows, $\Pi$ will be typically the uniform distribution in an orthonormal basis 
${\mathcal E}=\{E_1,\dots, E_{m^2}\}\subset {\mathbb H}_m,$ implying that 
$$
\|f_A\|_{L_2(\Pi)}^2 = \|A\|_{L_2(\Pi)}^2 = m^{-2}\|A\|_2^2,
$$
so, the $L_2(\Pi)$-norm is just a rescaled Hilbert--Schmidt norm.

Consider $A\in {\mathbb H}_m$ with spectral representation 
$A=\sum_{j=1}^{m'} \lambda_j P_j,$ $m'\leq m$ with distinct non-zero eigenvalues $\lambda_j.$
Denote by 
${\rm sign}(A):=\sum_{j=1}^{m'} {\rm sign}(\lambda_j) P_j$
and by ${\rm supp}(A)$ the linear span of the images of projectors $P_j, j=1,\dots, m'$
(the subspace ${\rm supp}(A)\subset {\mathbb C}^m$ will be called {\it the support} of $A$).

Given a subspace $L\subset {\mathbb C}^m,$ 
$L^{\perp}$ denotes the orthogonal complement of $L$ and  $P_L$ denotes the orthogonal projection onto $L.$ Let 
${\mathcal P}_L, {\mathcal P}_L^{\perp}$ be orthogonal projection operators 
in the space ${\mathbb H}_m$ (equipped with the Hilbert--Schmidt inner product),
defined as follows:
$$
{\mathcal P}_L^{\perp}(A)=P_{L^{\perp}}AP_{L^{\perp}},\ \ {\mathcal P}_L(A)=A-P_{L^{\perp}}AP_{L^{\perp}}.
$$
These two operators split any Hermitian matrix $A$ into two orthogonal parts, 
${\mathcal P}_L(A)$ and ${\mathcal P}_L^{\perp}(A),$ the first one being of rank
at most $2{\rm dim}(L).$ 

For a convex function $f:{\mathbb H}_m\mapsto {\mathbb R},$ $\partial f(A)$ denotes the subdifferential 
of $f$ at the point $A\in {\mathbb H}_m.$
It is well known that 
\begin{equation}
\label{subdiff}
\partial \|A\|_1 = \Bigl\{{\rm sign}(A)+ {\mathcal P}_L^{\perp}(M): M\in {\mathbb H}_m, 
\|M\|_{\infty}\leq 1\Bigr\},
\end{equation}
where $L={\rm supp}(A)$ (see \citealt{koltchinskii2011oracle}, p. 240 and references therein). 

$C, C_1, C',c, c',$ etc will denote constants (that do not depend on parameters of interest 
such as $m,n,$ etc) whose values could change from line to line (or, even, within the same line)
without further notice.  For nonnegative $A$ and $B,$ $A\lesssim B$ (equivalently, $B\gtrsim A$) means that $A\leq CB$ for some absolute constant $C>0,$ and $A\asymp B$ means that 
$A\lesssim B$ and $B\lesssim A.$ Sometimes, symbols $\lesssim, \gtrsim$ and $\asymp$
could be provided with subscripts (say, $A\lesssim_{\gamma}B$) to indicate that constant 
$C$ may depend on a parameter (say, $\gamma$).   

In what follows, $P$ denotes the distribution of $(X,Y)$ and $P_n$ denotes the corresponding 
empirical distribution based on the sample $(X_1,Y_1),\ldots,(X_n,Y_n)$ of $n$ i.i.d. observations.
Similarly, $\Pi$ is the distribution of $X$ (typically, uniform in an orthonormal basis) and $\Pi_n$
is the corresponding empirical distribution based on the sample $(X_1,\dots, X_n).$ 
We will use standard notations $Pf={\mathbb E}f(X,Y), P_n f=n^{-1}\sum_{j=1}^n f(X_j,Y_j)$
and $\Pi g={\mathbb E}g(X), P_n g=n^{-1}\sum_{j=1}^n g(X_j).$

\subsection{Estimation Methods}

Recall that the central problem in quantum state tomography is to estimate a large density 
matrix $\rho$ based on the data $(X_1,Y_1),\ldots,(X_n,Y_n)$ satisfying the trace regression 
model.  Often, the goal is to develop adaptive estimators with optimal dependence of the estimation error 
(measured by various statistically relevant distances) on the unknown rank of the target 
matrix $\rho$ under the assumption that $\rho$ is low rank, or on other complexity 
parameters in the case when the target matrix $\rho$ can be well approximated by low rank 
matrices.

The simplest estimation procedure for density matrix $\rho$ is the least squares estimator 
defined by the following convex optimization problem:
\begin{equation}
\label{least_squares}
 \hat{\rho}:=\underset{S\in\mathcal{S}_m}{\arg\min}\frac{1}{n}\sum\limits_{j=1}^n\left(Y_j-\left<S,X_j\right>\right)^2.
\end{equation}
Since, for all $S\in {\mathcal S}_m,$ $\|S\|_1=\text{tr}(S)=1,$
we have that
\begin{equation}
\label{trwnn}
\hat \rho= \hat{\rho}^{\eps}:=\underset{S\in\mathcal{S}_m}{\arg\min}\biggl[\frac{1}{n}\sum\limits_{j=1}^n\left(Y_j-\left<S,X_j\right>\right)^2+\eps\|S\|_1\biggr], \ \ \eps\geq 0.
\end{equation}
Thus, in the case of density matrices, the least squares estimator $\hat \rho$ coincides 
with the {\it matrix LASSO} estimator $\hat \rho^{\eps}$ with nuclear norm penalty and arbitrary 
value of regularization parameter $\eps.$ The nuclear norm penalty is used as a proxy 
of the rank that provides a convex relaxation for rank penalized least squares method. 
Matrix LASSO is a standard method of low rank estimation in trace regression models 
that has been intensively studied in the recent years, see, for instance, \cite{candes2011tight}, \cite{rohdetsybakov}, \cite{koltchinskii2011oracle}, \cite{koltchinskii2011nuclear}, \cite{negahban} and references therein. In the case of estimation of density matrices, due to their positive semidefiniteness
and trace constraint, the nuclear norm penalization is present implicitly even in the case of a non-penalized 
least squares estimator $\hat \rho$ (see also \citealt{koltchinskii2013remark}, \citealt{kalev2015informationally} where similar 
ideas were used).      

Note that the estimator $\hat \rho$ can be 
also rewritten as 
\begin{equation}
 \hat{\rho}:=\underset{S\in\mathcal{S}_m}{\arg\min}\biggl[\|S\|_{L_2(\Pi_n)}^2
 -\frac{2}{n}\sum_{j=1}^nY_j\big<S,X_j\big>\biggr].
\end{equation}
Replacing the empirical $\|\cdot\|_{L_2(\Pi_n)}$-norm with the ``true" $\|\cdot\|_{L_2(\Pi)}$-norm
(which could make sense in the case when the design distribution $\Pi$ is known) yields the following {\it modified least squares} estimator
studied in \cite{koltchinskii2011nuclear}, \cite{koltchinskii2013remark}: 
\begin{equation}
\label{KLT}
 \check{\rho}:=\underset{S\in\mathcal{S}_m}{\arg\min}\biggl[\|S\|_{L_2(\Pi)}^2-\frac{2}{n}\sum_{j=1}^nY_j\big<S,X_j\big>\biggr].
\end{equation}

Another estimator was proposed in \cite{koltchinskii2011neumann} 
and it is based on an idea of using so called {\it von Neumann entropy} as a penalizer
in least squares method. Von Neumann entropy is a canonical extension of Shannon's 
entropy to the quantum setting. For a density matrix $S\in {\mathcal S}_m,$ it is defined as 
${\mathcal E}(S):=-\text{tr}(S\log S).$ The estimator proposed in \cite{koltchinskii2011neumann}
is defined as follows
\begin{equation}
\label{trwvne}
\tilde{\rho}^{\eps}:=\underset{S\in\mathcal{S}_m}{\arg\min}
\biggl[\frac{1}{n}\sum\limits_{j=1}^n(Y_j-\big<S,X_j\big>)^2+\eps\text{tr}(S\log S)\biggr].
\end{equation}
Essentially, it is based on a trade-off between fitting the model via the least squares method in the class of all density matrices and maximizing the entropy of the quantum state. Note that (\ref{trwvne}) is also a convex optimization problem (due to concavity of von Neumann entropy, see \citealt{Nielsen2000}) and its solution $\tilde \rho^{\eps}$ is a full rank matrix (see
\citealt{koltchinskii2011neumann}, the proof of Proposition 3). It should be also mentioned 
that the idea of estimation of a density matrix of a quantum state by maximizing the 
von Neumann entropy subject to constraints based on the data has been used in quantum 
state tomography earlier (see \citealt{buvzek20046} and references therein).

\subsection{Distances between Density Matrices}

The main purpose of this paper is to study the optimality properties of estimator $\tilde{\rho}^{\epsilon}$ with respect to a variety of statistically meaningful distances,
in the case when the underlying density matrix $\rho$ is low rank. 
These distances include Schatten $p$-norm distances for $p\in [1,2],$\footnote{Similar problems for estimators $\hat \rho, \check \rho$
and for Schatten $p$-norm distances with  
$p\in (2,+\infty]$ are studied in a related paper by \cite{KoltchinskiiXia2015}}
but also quantum versions of Hellinger distance and Kullback-Leibler divergence that are 
of importance in quantum statistics and quantum information.  
A version of the (squared) Hellinger distance that will be studied is defined as
\begin{equation*}
 H^2(S_1,S_2):=2-2\text{tr}\sqrt{S_1^{\frac{1}{2}}S_2S_1^{\frac{1}{2}}}
\end{equation*}
for $S_1,S_2\in\mathcal{S}_m$ (see also \citealt{Nielsen2000}).  Clearly, $0\leq H^2(S_1,S_2)\leq 2.$ 
In quantum information literature, it is usually called Bures distance 
and it does not coincide with $\text{tr}(\sqrt{S_1}-\sqrt{S_2})^2$
(which is another possible non-commutative extension of the classical 
Hellinger distance). In fact, $H^2(S_1,S_2)\leq \text{tr}(\sqrt{S_1}-\sqrt{S_2})^2,
S_1,S_2\in {\mathcal S}_m,$ but the opposite inequality does not necessarily 
hold. The quantity $\text{tr}\sqrt{S_1^{\frac{1}{2}}S_2S_1^{\frac{1}{2}}}$ in the right
hand side of the definition of $H^2$ is a quantum version of Hellinger affinity. 

The noncommutative Kullback-Leibler divergence (or relative entropy distance)
$K(\cdot\|\cdot)$ is defined as (see also \citealt{Nielsen2000}): 
$$K(S_1\|S_2):=\big<S_1,\log S_1-\log S_2\big>.$$ 
If $\log S_2$ is not well-defined (for instance, some of the eigenvalues of $S_2$ are equal to $0$) we set $K(S_1\|S_2)=+\infty.$
The symmetrized version 
of Kullback-Leibler divergence is defined as 
$$
K(S_1;S_2):=K(S_1\|S_2)+K(S_2\|S_1)=\langle S_1-S_2, \log S_1-\log S_2\rangle.
$$ 

The following very useful inequality is a noncommutative extension of similar  
classical inequalities for total variation, Hellinger and Kullback-Leibler distances. 
It follows from representing the ``noncommutative distances" involved in the 
inequality as suprema of the corresponding classical distances between the distributions 
of outcomes of measurements for two states $S_1,S_2$ over all possible measurements
represented by positive operator valued measures  
(see, \citealt{Nielsen2000}, \citealt{Klauck2007}, \citealt{koltchinskii2011neumann}, Section 3 and references therein).
 
\begin{lemma}
For all $S_1,S_2\in\mathcal{S}_m,$ the following inequalities hold:
\begin{equation}
\label{khtracelem}
\frac{1}{4}\|S_1-S_2\|_1^2\leq H^2(S_1,S_2)\leq \big(K(S_1\|S_2)\wedge \|S_1-S_2\|_1\big).
\end{equation}
\end{lemma}

\subsection{Matrix Bernstein Inequalities}

Non-commutative (matrix) versions of Bernstein inequality will be used in what 
follows. The most common version is stated (in a convenient form for our applications) in the following lemma. 

\begin{lemma}
\label{matBernlem}
Let $X,X_1,\ldots,X_n\in\mathbb{H}_m$ be i.i.d. random matrices with $\mathbb{E}X=0,$ $\sigma_X^2:=\|\mathbb{E}X^2\|_{\infty}$ and $\|X\|_{\infty}\leq U$ a.s. for some $U>0.$ Then, for all $t\geq 0$ with probability at least $1-e^{-t},$
\begin{equation*}
 \biggl\|\frac{1}{n}\sum_{j=1}^nX_j\biggr\|_{\infty}\leq 2\biggl[\sigma_X\sqrt{\frac{t+\log(2m)}{n}}\bigvee U\frac{t+\log(2m)}{n}\biggr].
\end{equation*}
\end{lemma}

The proof of such bounds could be found, e.g., in \cite{tropp2012user}. Other versions 
on matrix Bernstein type inequalities for not necessarily bounded random matrices will be also used in what follows and they could be found in \cite{koltchinskii2011oracle}, \cite{koltchinskii2013remark}. A simple consequence of the inequality of Lemma 
\ref{matBernlem} is the following expectation bound:
$$
{\mathbb E}\biggl\|\frac{1}{n}\sum_{j=1}^nX_j\biggr\|_{\infty}
\lesssim \biggl[\sigma_X\sqrt{\frac{\log(2m)}{n}}\bigvee U\frac{\log(2m)}{n}\biggr].
$$
It follows from the exponential bound by integrating the tail probabilities.

The paper is organized as follows. 
In Section~\ref{minimaxsec}, minimax lower bounds on estimation error of low rank 
density matrices are provided in Schatten $p$-norm, Hellinger (Bures)
and Kullback-Leibler distances. In Section~\ref{vonneumann_1}, sharp low rank oracle inequalities for
von Neumann entropy penalized least squares estimator are derived in the case of trace 
regression model with bounded response. In Section~\ref{vonneumann_2},  low rank 
oracle inequalities are established in the case of trace regression with Gaussian noise.
In addition to this, in these two sections, upper bounds on estimation error with respect to Kullback-Leibler 
distance are obtained. In Section \ref{tildesec}, they are further developed and extended to other distances (Hellinger distance, 
Schatten $p$-norm distances for $p\in [1,2]$) showing the minimax optimality (up to logarithmic factors)
of the error rates of the least squares estimator with von Neumann entropy penalization.

\section{Minimax Lower Bounds}\label{minimaxsec}

In this section, we provide main results on the minimax lower bounds on the risk of estimation of density matrices
with respect to Schatten $p$-norm (or, rather $q$-norm in the notations used below) distances as well as Hellinger-Bures distance and Kullback-Leibler divergence.  

Minimax lower bounds will be derived for the class 
${\mathcal S}_{r,m}:=\{S\in {\mathcal S}_m: {\rm rank}(S)\leq r\}$ consisting of all density matrices 
of rank at most $r$ (the low rank case). 
We will start with the case of trace regression with Gaussian noise. Given that the sample 
$(X_1,Y_1), \dots, (X_n,Y_n)$
satisfies Assumption \ref{Gaussian_noise} with the target density matrix $\rho\in {\mathcal S}_m$ 
and noise variance $\sigma_{\xi}^2,$ let ${\mathbb P}_{\rho}$
denote the corresponding probability distribution.

Note that \cite{ma2013volume} developed a method of deriving minimax lower bounds for distances based on unitary invariant norms, including Schatten $p$-norms in matrix problems, and obtained such lower bounds, in particular,  in matrix completion problem.
The approach used in our paper is somewhat different and the aim is to develop such bounds under an additional constraint 
that the target matrix is a density matrix.
The resulting bounds are also somewhat different, they involve an additional 
term that does not depend on the rank, but does depend on $q.$ Essentially, it means that the ``complexity" of the problem 
is controlled by a ``truncated rank" $r \wedge \frac{1}{\tau},$ where $\tau=\frac{\sigma_{\xi}m^{3/2}}{\sqrt{n}}$ rather than 
by the actual rank $r.$ The upper bounds of Section \ref{tildesec} show 
that such a structure of the bound is, indeed, necessary. 
It should be also mentioned that minimax lower bounds on the nuclear norm error 
of estimation of density matrices have been obtained earlier in \cite{flammia2012quantum}
(see Remark \ref{Gross_Flammia} below).

\begin{theorem}
 \label{minmaxthm1}
For all $q\in [1,+\infty],$ there exist constants $c,c'>0$ such that,  the following bounds hold:  
\begin{equation}
 \label{minmaxthm1boundq}
 \underset{\hat{\rho}}{\inf}\underset{\rho\in\mathcal{S}_{r,m}}{\sup}\mathbb{P}_{\rho}
 \biggl\{\|\hat{\rho}-\rho\|_q\geq c\biggl(\frac{\sigma_{\xi}m^{\frac{3}{2}}r^{1/q}}{\sqrt{n}}\bigwedge 
 \biggl(\frac{\sigma_{\xi}m^{3/2}}{\sqrt{n}}\biggr)^{1-\frac{1}{q}}\bigwedge 1\biggr)\biggr\}\geq c',
\end{equation}
 \begin{equation}
 \label{minmaxthm1boundH}
 \underset{\hat{\rho}}{\inf}\underset{\rho\in\mathcal{S}_{r,m}}{\sup}\mathbb{P}_{\rho}\biggl\{H^2(\hat{\rho},\rho)\geq c\biggl(\frac{\sigma_{\xi}m^{\frac{3}{2}}r}{\sqrt{n}}\bigwedge 1\biggr)\biggr\}\geq c',
 \end{equation}
and
 \begin{equation}
 \label{minmaxthm1boundK}
 \underset{\hat{\rho}}{\inf}\underset{\rho\in\mathcal{S}_{r,m}}{\sup}\mathbb{P}_{\rho}\biggl\{K(\rho\|\hat{\rho})\geq c\biggl(\frac{\sigma_{\xi}m^{\frac{3}{2}}r}{\sqrt{n}}\bigwedge 1\biggr)\biggr\}\geq c',
 \end{equation}
 where $\inf_{\hat{\rho}}$ denotes the infimum over all estimators $\hat{\rho}$ in $\mathcal{S}_m$
 based on the data $(X_1,Y_1), \dots, (X_n,Y_n)$ satisfying the Gaussian trace regression 
 model with noise variance $\sigma_{\xi}^2.$
\end{theorem}

\begin{proof} 
A couple of preliminary facts will be needed in the proof. We start with bounds on the packing numbers 
of Grassmann manifold $\mathcal{G}_{k,l},$ which is the set of all $k$-dimensional subspaces $L$ of the $l$-dimensional space $\mathbb{R}^l.$ Given such a subspace $L\subset \mathbb{R}^l$ with ${\rm dim}(L)=k,$
let $P_L$ be the orthogonal projection onto $L$ and let ${\mathfrak P}_{k,l}:=\{P_L: L\in {\mathcal G}_{k,l}\}.$
The set of all $k$-dimensional projectors  ${\mathfrak P}_{k,l}$ will be equipped with Schatten $q$-norm distances 
for all $q\in [1,+\infty]$ (which also could be viewed as distances on the Grassmannian itself):
$d_q(Q_1,Q_2):=\|Q_1-Q_2\|_q, Q_1,Q_2\in {\mathfrak P}_{k,l}.$ 
Recall that the $\eps$-\textit{packing number} of a metric space $(T,d)$ is defined as
\begin{equation*}
 D(T,d,\eps)=\max\Big\{n: \textrm{there are } t_1,\ldots,t_n\in T, \textrm{such that }\underset{i\neq j}{\min}\ d(t_i,t_j)>\eps\Big\}.
\end{equation*}

The following lemma (see \citealt[Proposition 8]{pajor1998metric})
will be used to control the packing numbers of ${\mathfrak P}_{k,l}$ with respect 
to Schatten distances $d_q.$

\begin{lemma}
\label{pajorlem}
 For all integer $1\leq k\leq l$ such that $k\leq l-k$, and all $1\leq q\leq\infty,$
 the following bounds hold
 \begin{equation}
 \left(\frac{c}{\eps}\right)^d\leq D({\mathfrak P}_{k,l},d_q,\eps k^{1/q})\leq \left(\frac{C}{\eps}\right)^d,\ \eps>0
 \end{equation}
with $d=k(l-k)$ and universal positive constants $c,C.$
\end{lemma}

In addition to this, we need the following well known information-theoretic bound frequently 
used in derivation of minimax lower bounds (see \citealt[Theorem 2.5]{intro}).
Let $\Theta=\{\theta_0,\theta_1,\dots, \theta_M\}$ be a finite parameter space equipped 
with a metric $d$ and let ${\mathcal P}:=\{{\mathbb P}_{\theta}:\theta \in \Theta\}$ be 
a family of probability distributions in some sample space.
Given ${\mathbb P}, {\mathbb Q}\in {\mathcal P},$ let 
$K({\mathbb P}\|{\mathbb Q}):={\mathbb E}_{\mathbb P}\log\frac{d{\mathbb P}}{d{\mathbb Q}}$
be the Kullback-Leibler divergence between ${\mathbb P}$ and ${\mathbb Q}.$ 

\begin{proposition}
\label{fanothm}
 Suppose that the following conditions hold:
 \begin{enumerate}[(i)]
 \item for some $s>0,$ $d(\theta_j,\theta_k)\geq 2s>0, 0\leq j<k\leq M;$
 \item for some $0<\alpha<1/8,$ 
 $\frac{1}{M}\sum\limits_{j=1}^MK(\mathbb{P}_{\theta_j}\|\mathbb{P}_{\theta_0})\leq \alpha\log M$
 \end{enumerate}
 Then, for a positive constant $c_{\alpha},$
 \begin{equation*}
 \underset{\hat{\theta}}{\inf}\underset{\theta\in\Theta}{\sup} \mathbb{P}_{\theta}\{d(\hat{\theta},\theta)\geq s\}\geq c_{\alpha},
 \end{equation*}
 where the infimum is taken over all estimators $\hat \theta\in \Theta$ based on an observation 
sampled from ${\mathbb P}_{\theta}.$ 
\end{proposition}

We now turn to the actual proof of Theorem \ref{minmaxthm1}.
Under Assumption \ref{Gaussian_noise}, 
the following computation is well known: for $\rho_1, \rho_2\in {\mathcal S}_{r,m},$
\begin{equation}
\label{kldiv}
\begin{split}
 K(\mathbb{P}_{\rho_1}\|\mathbb{P}_{\rho_2})&=\mathbb{E}_{\mathbb{P}_{\rho_1}}\log\frac{\mathbb{P}_{\rho_1}}{\mathbb{P}_{\rho_2}}\biggl(X_1,Y_1,\ldots,X_n,Y_n\biggr)\\
&=\mathbb{E}_{\mathbb{P}_{\rho_1}}\sum\limits_{j=1}^n\Big[-\frac{(Y_j-\left<\rho_1,X_j\right>)^2}{2\sigma_{\xi}^2}+\frac{(Y_j-\left<\rho_2,X_j\right>)^2}{2\sigma_{\xi}^2} \Big]\\
&=\mathbb{E}\sum\limits_{j=1}^n \frac{\left<\rho_1-\rho_2,X_j\right>^2}{2\sigma_\xi^2}=\frac{n}{2\sigma_{\xi}^2}\|\rho_1-\rho_2\|_{L_2(\Pi)}^2.
\end{split}
\end{equation}

It is enough to prove the bounds for $2\leq r\leq m/2.$ The proof in the case $r=1$ is simpler and the case $r>m/2$ easily reduces to the case $r\leq m/2.$ We will use Lemma \ref{pajorlem} to construct a well separated (with respect to $d_q$) subset of density 
matrices in ${\mathcal S}_{r,m}.$ To this end, first choose a subset ${\mathcal D}_q\subset \mathfrak{P}_{r-1,m-1}$
such that ${\rm card}({\mathcal D}_q)\geq 2^{(r-1)(m-r)}$ and, for some constant $c',$ 
$\|Q_1-Q_2\|_q \geq c' (r-1)^{1/q},$
$Q_1,Q_2\in \mathfrak{P}_{r-1,m-1}, Q_1\neq Q_2.$ Such a choice is possible due to the lower bound on the packing numbers of Lemma \ref{pajorlem}. For $Q\in {\mathcal D}_q$ (note that $Q$ can be viewed as an $(m-1)\times (m-1)$
matrix with real entries) and $\kappa\in (0,1),$ consider the following $m\times m$ matrix 
\begin{equation}
\label{construct1}
 S=S_Q=\left(
 \begin{array}{cc}
  1-\kappa&\bf{0}'\\
\bf{0}&\kappa \frac{Q}{r-1}
 \end{array}
 \right).
\end{equation}
Note that $S$ is symmetric positively-semidefinite real matrix of unit trace. It is straightforward 
to check that it defines a Hermitian positively-semidefinite operator in ${\mathbb C}^m$ of unit trace,
and it can be identified with a density matrix $S\in {\mathcal S}_m.$ Clearly, $S$ is of rank $r,$ so,
$S\in {\mathcal S}_{r,m}.$ 

We will take $\kappa:= c_1\frac{\sigma_{\xi}m^{3/2}(r-1)}{\sqrt{n}}$ with a small enough absolute constant $c_1>0$ and first assume that $\kappa<1$ (as it is needed in definition Equation \ref{construct1}).

Let ${\mathcal S}_q' := \{S_Q: Q\in {\mathcal D}_q\}$  and consider a family of 
$M+1={\rm card}({\mathcal D}_q)\geq 2^{(r-1)(m-r)}$ distributions $\{{\mathbb P}_S: S\in {\mathcal S}_q'\}.$
It is immediate 
that for $S_1=S_{Q_1}, S_2=S_{Q_2},$ $Q_1,Q_2\in {\mathcal D}_q, Q_1\neq Q_2,$ we have 
\begin{equation}
\label{separ}
\begin{split}
&
\|S_1-S_2\|_q=\frac{\kappa}{r-1} \|Q_1-Q_2\|_q\geq c'\kappa (r-1)^{1/q-1}\\
&\geq  c' c_1\frac{\sigma_{\xi}m^{3/2}(r-1)^{1/q}}{\sqrt{n}}\geq 
c\frac{\sigma_{\xi}m^{3/2}r^{1/q}}{\sqrt{n}}
\end{split}
\end{equation}
with some constant $c>0,$ implying condition (i) of Proposition \ref{fanothm}
with $s=\frac{c}{2}\frac{\sigma_{\xi}m^{3/2}r^{1/q}}{\sqrt{n}}.$

We will now check its condition (ii) . In view of (\ref{kldiv}), we have, for all 
$S_1=S_{Q_1},S_2=S_{Q_2}\in {\mathcal S}_q',$ 
\begin{equation} 
\label{fanothm2}
\begin{split}
K(\mathbb{P}_{S_1}\|\mathbb{P}_{S_2})&=\frac{n}{2\sigma_{\xi}^2}\|S_1-S_2\|_{L_2(\Pi)}^2=
\frac{n}{2\sigma_{\xi}^2m^2}\|S_1-S_2\|_2^2\\
&=\frac{n\kappa^2}{2\sigma_{\xi}^2m^2 (r-1)^2}\|Q_1-Q_2\|_2^2
\leq\frac{4n(r-1)\kappa^2}{2\sigma_{\xi}^2m^2(r-1)^2}=2c_1^2m(r-1)\\	
&\leq\alpha m(r-1)/\log(2)/4\leq \frac{\alpha}{2} (r-1)(m-r)\log(2)\leq \alpha\log M,
\end{split}
\end{equation}
provided that constant $c_1$ is small enough, so, condition (ii) of Proposition \ref{fanothm}
is also satisfied. 
Proposition \ref{fanothm} implies that, under the assumption $\kappa=c_1\frac{\sigma_{\xi}m^{3/2}(r-1)}{\sqrt{n}}<1,$
the following minimax lower bound holds for some $c, c'>0:$
\begin{equation}
\label{lower_first}
\underset{\hat{\rho}}{\inf}\underset{\rho\in\mathcal{S}_{r,m}}{\sup}\mathbb{P}_{\rho}
\biggl\{\|\hat{\rho}-\rho\|_q\geq c\frac{\sigma_{\xi}m^{\frac{3}{2}}r^{1/q}}{\sqrt{n}}\biggr\}\geq c'.
\end{equation}
In the case when 
$$ 
c_1\frac{\sigma_{\xi}m^{3/2}}{\sqrt{n}}<1\leq c_1\frac{\sigma_{\xi}m^{3/2}(r-1)}{\sqrt{n}},
$$
one can choose $2\leq r'<r-1$ such that, for some constant $c_2>0,$ 
$$c_2<c_1\frac{\sigma_{\xi}m^{3/2}(r'-1)}{\sqrt{n}}<1.$$
For such a choice of $r',$ it follows from (\ref{lower_first}) that 
\begin{equation}
\label{lower_first''}
\underset{\hat{\rho}}{\inf}\underset{\rho\in\mathcal{S}_{r',m}}{\sup}\mathbb{P}_{\rho}
\biggl\{\|\hat{\rho}-\rho\|_q\geq c\frac{\sigma_{\xi}m^{\frac{3}{2}}(r')^{1/q}}{\sqrt{n}}\biggr\}\geq c'.
\end{equation}
The definition of $r'$ implies that  
$$
r'\asymp r'-1\asymp \biggl(\frac{\sigma_{\xi}m^{3/2}}{\sqrt{n}}\biggr)^{-1}.
$$
Therefore, 
$$
\frac{\sigma_{\xi}m^{\frac{3}{2}}(r')^{1/q}}{\sqrt{n}}\asymp \biggl(\frac{\sigma_{\xi}m^{3/2}}{\sqrt{n}}\biggr)^{1-1/q},
$$
and, since ${\mathcal S}_{r',m}\subset {\mathcal S}_{r,m},$ bound (\ref{lower_first''}) yields 
\begin{equation}
\label{lower_second}
\underset{\hat{\rho}}{\inf}\underset{\rho\in\mathcal{S}_{r,m}}{\sup}\mathbb{P}_{\rho}
\biggl\{\|\hat{\rho}-\rho\|_q\geq c\biggl(\frac{\sigma_{\xi}m^{3/2}}{\sqrt{n}}\biggr)^{1-1/q}\biggr\}
\geq 
\underset{\hat{\rho}}{\inf}\underset{\rho\in\mathcal{S}_{r',m}}{\sup}\mathbb{P}_{\rho}
\biggl\{\|\hat{\rho}-\rho\|_q\geq c\biggl(\frac{\sigma_{\xi}m^{3/2}}{\sqrt{n}}\biggr)^{1-1/q}\biggr\}\geq c'
\end{equation}
for some constants $c,c'>0.$
This allows us to recover the second term in the minimum in bound (\ref{minmaxthm1boundq}).
Finally, in the case when 
$
c_1\frac{\sigma_{\xi}m^{3/2}}{\sqrt{n}}>1,
$
the minimax lower bound becomes a constant 
(and the proof is based on a simplified version of the above argument that could be done 
for $r=1$). This completes the proof of bound (\ref{minmaxthm1boundq}) for Schatten $q$-norms.

The proof of bound (\ref{minmaxthm1boundH}) for the Hellinger distance is similar. 
In the case $r\geq 2,$ we will use a ``well separated" set of density matrices ${\mathcal S}_q'\subset {\mathcal S}_{r,m}$
for $q=1$ constructed above. We still use $\kappa:= c_1\frac{\sigma_{\xi}m^{3/2}(r-1)}{\sqrt{n}}$
assuming first that $\kappa\in (0,1).$  
For $S_{Q_1}, S_{Q_2}\in {\mathcal S}_q'$ with $Q_1\neq Q_2,$
it follows by a simple computation and using bound (\ref{khtracelem}) that, for some $c''>0,$ 
\begin{equation*}
\begin{split}
&
H^2(S_{Q_1}, S_{Q_2})= \kappa H^2\Bigl(\frac{Q_1}{r-1},\frac{Q_2}{r-1}\Bigr)
\\
&
\geq \frac{1}{4}\frac{\kappa}{(r-1)^2} \|Q_1-Q_2\|_1^2\geq \frac{(c')^2}{4} \kappa \geq c''\frac{\sigma_{\xi}m^{3/2}(r-1)}{\sqrt{n}}.
\end{split}
\end{equation*}
Repeating the argument based on Proposition \ref{fanothm} yields bound (\ref{minmaxthm1boundH}) in the 
case when $\kappa= c_1\frac{\sigma_{\xi}m^{3/2}(r-1)}{\sqrt{n}}<1,$ and in the opposite case it is easy 
to see that the lower bound is a constant.  
 
Finally, bound (\ref{minmaxthm1boundK}) for the Kullback--Leibler divergence follows from (\ref{minmaxthm1boundH})
and the inequality $K(\rho\|\hat \rho)\geq H^2(\hat \rho, \rho)$ (see inequality \ref{khtracelem}).
 
\end{proof}

Next we state similar results in the case of trace regression model with bounded response (see Assumption \ref{bounded_response}).
Denote by ${\mathcal P}_{r,m}(\bar U)$ 
the class of all distributions $P$
of $(X,Y)$ such that Assumption \ref{bounded_response} holds for some $\bar U$ 
and ${\mathbb E}(Y|X)=\langle \rho_P,X\rangle$ for some $\rho_P\in {\mathcal S}_{r,m}.$
Given $P,$ ${\mathbb P}_P$ denotes the corresponding 
probability measure (such that $(X_1,Y_1),\ldots,(X_n,Y_n)$ are i.i.d. copies of $(X,Y)$ sampled from $P$).

\begin{theorem}
\label{minmaxthm3}
Suppose $\bar U\geq 2U.$
For all $q\in [1,+\infty],$ there exist absolute constants $c,c'>0$ such that the following bounds hold:
\begin{equation}
 \label{minmaxthm1boundq_U}
 \underset{\hat{\rho}}{\inf}\underset{P\in\mathcal{P}_{r,m}(\bar U)}{\sup}\mathbb{P}_{P}
 \biggl\{\|\hat{\rho}-\rho_P\|_q\geq c\biggl(\frac{\bar U m^{\frac{3}{2}}r^{1/q}}{\sqrt{n}}\bigwedge \biggl(\frac{\bar U m^{3/2}}{\sqrt{n}}\biggr)^{1-\frac{1}{q}}\bigwedge 1\biggr)\biggr\}\geq c',
\end{equation}
 \begin{equation}
 \label{minmaxthm1boundH_U}
 \underset{\hat{\rho}}{\inf}\underset{P\in\mathcal{P}_{r,m}(\bar U)}{\sup}\mathbb{P}_{P}\biggl\{H^2(\hat{\rho},\rho_P)\geq c\biggl(\frac{\bar U m^{\frac{3}{2}}r}{\sqrt{n}}\bigwedge 1\biggr)\biggr\}\geq c',
 \end{equation}
and
 \begin{equation}
 \label{minmaxthm1boundK_U}
 \underset{\hat{\rho}}{\inf}\underset{P\in\mathcal{P}_{r,m}(\bar U)}{\sup}\mathbb{P}_{P}\biggl\{K(\rho_P\|\hat{\rho})\geq c\biggl(\frac{\bar U m^{\frac{3}{2}}r}{\sqrt{n}}\bigwedge 1\biggr)\biggr\}\geq c',
 \end{equation}
 where $\inf_{\hat{\rho}}$ denotes the infimum over all estimators $\hat{\rho}$ in $\mathcal{S}_m$
 based on the data $(X_1,Y_1), \dots, (X_n,Y_n).$ 
\end{theorem}


\begin{proof}
The proof relies on an idea already used in a context of matrix completion by \cite{koltchinskii2011nuclear} (see their Theorem 7).
We need the same family ${\mathcal S}_q'\subset {\mathcal S}_{r,m}$ 
of ``well separated" density matrices of rank $r$ as in the proof of 
Theorem  \ref{minmaxthm1}. 
For a density matrix $\rho,$
let $(X,Y)$ be a random couple such that $X$ is sampled from 
the uniform distribution $\Pi$ in ${\mathcal E}$ and, conditionally on 
$X,$ $Y$ takes value $+\bar U$ with probability $p_{\rho}(X):=\frac{1}{2}+\frac{\langle \rho,X\rangle}{2\bar U}$
and value $-\bar U$ with probability $q_{\rho}(X):=\frac{1}{2}-\frac{\langle \rho,X\rangle}{2\bar U}.$
Since $\bar U\geq 2U$ and $|\langle \rho,X\rangle|\leq \|\rho\|_1\|X\|_{\infty}\leq U,$
we have $p_{\rho}(X), q_{\rho}(X)\in [1/4, 3/4]$ (so, they are bounded away
from $0$ and from $1$). Clearly, ${\mathbb E}_{\rho}(Y|X)=\langle \rho,X\rangle.$
Let $P_{\rho}$ denote the distribution of such a 
couple and ${\mathbb P}_{\rho}$ denote the corresponding distribution 
of the data $(X_1,Y_1), \dots, (X_n,Y_n).$ Then, for all $\rho\in {\mathcal S}_{r,m},$
$P_{\rho}\in \mathcal{P}_{r,m}(\bar U).$ The only difference with the proof of 
Theorem  \ref{minmaxthm1} is in the bound on Kullback-Leibler divergence $K({\mathbb P}_{\rho_1}\|{\mathbb P}_{\rho_2})$
(see Equation \ref{kldiv}). It is easy to see that  
\begin{equation}
\label{KL-binary}
K(\mathbb{P}_{\rho_1}\|\mathbb{P}_{\rho_2})=n\mathbb{E}\left(p_{\rho_1}(X)\log\frac{p_{\rho_1}(X)}{p_{\rho_2}(X)}
+q_{\rho_1}(X)\log\frac{q_{\rho_1}(X)}{q_{\rho_2}(X)}\right).
\end{equation}
The following simple inequality will be used: for all $a,b\in [1/4,3/4],$ 
\begin{equation*}
 a\log\frac{a}{b}+(1-a)\log\frac{1-a}{1-b}\leq 12(a-b)^2.
\end{equation*}
It implies that 
\begin{equation*}
 K(\mathbb{P}_{\rho_1}\|\mathbb{P}_{\rho_2})\leq 3n\mathbb{E}\frac{\left<\rho_1-\rho_2,X\right>^2}{\bar U^2}\leq \frac{3n}{\bar U^2}\|\rho_1-\rho_2\|_{L_2(\Pi)}^2.
\end{equation*}
This bound is used instead of identity (\ref{kldiv}) from the proof of Theorem~\ref{minmaxthm1}.
The rest of the proof is the same.
\end{proof}

Note that the proof requires the possible range $[-\bar U,\bar U]$ of response variable $Y$
to be larger than the possible range $[-U,U]$ of Fourier coefficients $\langle \rho,E_j\rangle, j=1,\dots, m^2.$
This is not the case for standard QST model described in the introduction (see also the example of Pauli 
measurements) and it is of interest to prove a version of minimax lower bounds without this constraint, 
including the case when $\bar U=U.$ The following theorem is a result in this direction.

\begin{theorem}
\label{minmaxthm3'''}
Suppose Assumption \ref{orthonorm} is satisfied and, moreover, 
for some constant $\gamma \in (0,1),$
\begin{equation}
\label{traceEj}
\Bigl|{\rm tr}(E_k)\Bigr|\leq (1-\gamma)U m,\ k=1,\dots, m^2.
\end{equation}
Then, for all $q\in [1,+\infty],$ there exist constants $c_{\gamma},c_{\gamma}'>0$ such that the following bounds hold:
\begin{equation}
 \label{minmaxthm1boundq_U'''}
 \underset{\hat{\rho}}{\inf}\underset{P\in\mathcal{P}_{r,m}(U)}{\sup}\mathbb{P}_{P}
 \biggl\{\|\hat{\rho}-\rho_P\|_q\geq c_{\gamma}\biggl(\frac{U m^{\frac{3}{2}}r^{1/q}}{\sqrt{n}}\bigwedge \biggl(\frac{U m^{3/2}}{\sqrt{n}}\biggr)^{1-\frac{1}{q}}\bigwedge 1\biggr)\biggr\}\geq c_{\gamma}',
\end{equation}
 \begin{equation}
 \label{minmaxthm1boundH_U'''}
 \underset{\hat{\rho}}{\inf}\underset{P\in\mathcal{P}_{r,m}(U)}{\sup}\mathbb{P}_{P}\biggl\{H^2(\hat{\rho},\rho_P)\geq c_{\gamma}\biggl(\frac{U m^{\frac{3}{2}}r}{\sqrt{n}}\bigwedge 1\biggr)\biggr\}\geq c_{\gamma}',
 \end{equation}
and
 \begin{equation}
 \label{minmaxthm1boundK_U'''}
 \underset{\hat{\rho}}{\inf}\underset{P\in\mathcal{P}_{r,m}(U)}{\sup}\mathbb{P}_{P}\biggl\{K(\rho_P\|\hat{\rho})\geq c_{\gamma}\biggl(\frac{U m^{\frac{3}{2}}r}{\sqrt{n}}\bigwedge 1\biggr)\biggr\}\geq c_{\gamma}',
 \end{equation}
 where $\inf_{\hat{\rho}}$ denotes the infimum over all estimators $\hat{\rho}$ in $\mathcal{S}_m$
 based on the data $(X_1,Y_1), \dots, (X_n,Y_n).$ 
\end{theorem}

\begin{proof}
The proof is based on the following lemma:

\begin{lemma}
\label{rad}
Suppose assumption (\ref{traceEj}) holds. 
Let $K$ be a sufficiently large absolute constant (to be chosen later) and let 
$m$ satisfy the condition $K\frac{\log m}{\sqrt{m}}\leq \frac{\gamma}{2}$
(which means that $m\geq A_{\gamma}$ for some constant $A_{\gamma}$). 
Then there exists $v\in {\mathbb C}^m$ with $\|v\|=1$ such that 
\begin{equation}
\label{q-form}
\Bigl|\langle E_k v,v\rangle\Bigr|\leq (1-\gamma/2)U, k=1,\dots, m^2.
\end{equation}
\end{lemma}

\begin{proof}
We will prove this fact by a probabilistic argument. 
Namely, set $v:=m^{-1/2}(\eps_1,\dots, \eps_m),$ where $\eps_j=\pm 1.$
We will show that there is a random choice of ``signs" $\eps_j$ such that (\ref{q-form})
holds. Assume that $\eps_j, j=1,\dots, m$ are i.i.d. and take values $\pm 1$ with probability 
$1/2$ each. Let $E_k:=(a_{ij}^{(k)})_{i,j=1,\dots,m}.$ For simplicity, assume that 
$(a_{ij}^{(k)})_{i,j=1,\dots, m}$ is a symmetric real matrix (in the complex case, the proof can be easily modified). 
We have 
$$
\langle E_k v,v\rangle = \frac{1}{m}\sum_{i=1}^m a_{ii}^{(k)}\eps_i^2+ \frac{1}{m}\sum_{i\neq j}a_{ij}^{(k)}\eps_i \eps_j
= \frac{{\rm tr}(E_k)}{m}+ \frac{1}{m}\sum_{i\neq j}a_{ij}^{(k)}\eps_i \eps_j.
$$
It is well known that 
$$
{\rm Var}\biggl(\sum_{i\neq j}a_{ij}^{(k)}\eps_i \eps_j\biggr)= 
{\mathbb E}\biggl(\sum_{i\neq j}a_{ij}^{(k)}\eps_i \eps_j\biggr)^2
=2\sum_{i\neq j}\Bigl(a_{ij}^{(k)}\Bigr)^2
\leq 2\sum_{i,j}\Bigl(a_{ij}^{(k)}\Bigr)^2=2\|E_k\|_2^2=2.
$$
Moreover, it follows from exponential inequalities for Rademacher chaos (see, e.g., Corollary 3.2.6 in \citealt{Gine}) that for some absolute constant $K>0$ and for all $t>0,$ with probability at least $1-e^{-t}$
$$
\biggl|\langle E_k v,v\rangle-\frac{{\rm tr}(E_k)}{m}\biggr|= \biggl|\frac{1}{m}\sum_{i\neq j}a_{ij}^{(k)}\eps_i \eps_j\biggr|
\leq \frac{K t}{m}.
$$
Taking $t=2\log m$ and using the union bound, we conclude that with probability at least $1-m e^{-2\log m}=1-\frac{1}{m}>0,$
$$
\max_{1\leq k\leq m^2}
\biggl|\langle E_k v,v\rangle-\frac{{\rm tr}(E_k)}{m}\biggr|
\leq \frac{K \log m}{m} \leq \frac{K\log m}{\sqrt{m}}U\leq \frac{\gamma}{2}U,
$$
where we also used the fact that $U\geq m^{-1/2}.$
Thus, there exists a choice of signs $\eps_j$ such that 
$$
\max_{1\leq k\leq m^2}\Bigl|\langle E_k v,v\rangle\Bigr|
\leq \max_{1\leq k\leq m}\biggl|\frac{{\rm tr}(E_k)}{m}\biggr|+\frac{\gamma}{2}U,
$$
which, under condition (\ref{traceEj}), implies (\ref{q-form}).
\end{proof}

We set $e_1:=v$ (where $v$ is the unit vector introduced in Lemma \ref{rad}) and construct an orthonormal basis 
$e_1,\dots, e_m.$ Assume that matrices $S_Q$ defined by (\ref{construct1}) represent linear transformations 
in basis $e_1,\dots, e_m.$ Then we have 
$$
\langle S_Q, E_k\rangle = (1-\kappa)\langle E_k v,v\rangle + \frac{\kappa}{r-1} \langle Q,E_k\rangle.
$$
Therefore,
$$
\Bigl|\langle S_Q, E_k\rangle\Bigr|
\leq (1-\kappa)\Bigl|\langle E_k v,v\rangle\Bigr|
+ \frac{\kappa}{r-1} \|E_k\|_{\infty}\|Q\|_1 
\leq (1-\kappa)(1-\gamma/2)U + \kappa U
=(1-(1-\kappa)(\gamma/2))U.
$$
Assuming that $\kappa\leq 1/2,$ we get 
\begin{equation}
\label{SQ}
\Bigl|\langle S_Q, E_k\rangle\Bigr|
\leq (1-\gamma/4)U,\ k=1,\dots, m^2. 
\end{equation}
The rest of the proof becomes similar to the proof of Theorem \ref{minmaxthm3} (with $\bar U=U$). 
Namely, bound (\ref{SQ}) implies that, for $\rho=S_Q$ and $X$ being sampled from the orthonormal 
basis $\{E_1,\dots, E_{m^2}\},$ probabilities $p_{\rho}(X)$ and $q_{\rho}(X)$ are 
bounded away from $0$ and from $1:$ $p_{\rho}(X), q_{\rho}(X)\in [\gamma/8, 1-\gamma/8].$ 
This allows us to complete the argument of the proof of Theorem \ref{minmaxthm3}.
\end{proof}

Theorem \ref{minmaxthm3'''} does not apply directly to the Pauli basis since condition (\ref{traceEj}) fails 
in this case. Indeed, by the definition of Pauli basis, $U=m^{-1/2}$ and ${\rm tr}(E_1)=\sqrt{m}=Um >(1-\gamma)Um.$ Note also that ${\rm tr}(E_j)=0, j=2,\dots, m^2.$ 
Thus, for Pauli basis, $E_1$ is the only matrix for which 
condition (\ref{traceEj}) fails. However, for this matrix 
$\langle \rho, E_1\rangle= m^{-1/2}{\rm tr}(\rho)=m^{-1/2}=U$ for all density matrices $\rho\in {\mathcal S}_m.$
This immediately implies that $p_{\rho}(E_1)=1$ and $q_{\rho}(E_1)=0$ for all $\rho\in {\mathcal S}_m$
and, as a result, the value $X=E_1$ does not have an impact on the computation of Kullback-Leibler 
divergence in (\ref{KL-binary}). 
For the rest of the matrices in the Pauli basis, condition (\ref{traceEj})
holds implying also bound (\ref{q-form}). 
Therefore, if $X\neq E_1,$ we still have that, for $\rho=S_Q,$ $p_{\rho}(X), q_{\rho}(X)\in [\gamma/8, 1-\gamma/8],$ 
and the proof of Theorem \ref{minmaxthm3} can be completed in this case, too. Note also that, given $X$ sampled 
from the Pauli basis, the binary random variable $Y$ taking values $\pm U=\pm \frac{1}{\sqrt{m}}$ with 
probabilities $p_{\rho}(X)$ and $q_{\rho}(X),$ respectively (this is exactly the random variable used 
in the construction of the proof of Theorem \ref{minmaxthm3}) coincides with an outcome of a Pauli 
measurement for the system prepared in state $\rho.$ These considerations yield the following minimax 
lower bounds for Pauli measurements. 

\begin{theorem}
\label{minmaxthm3_Pauli}
Let $\{E_1,\dots, E_{m^2}\}$ be the Pauli basis in the space ${\mathbb H}_m$ of $m\times m$ Hermitian 
matrices and let $X_1,\dots, X_n$ be i.i.d. random variables sampled from the uniform distribution 
in $\{E_1,\dots, E_{m^2}\}.$ Let $Y_1,\dots, Y_n$ be outcomes of measurements of observables $X_1,\dots, X_n$
for the system being identically prepared $n$ times in state $\rho.$ The corresponding distribution 
of the data $(X_1,Y_1), \dots, (X_n,Y_n)$ will be denoted by ${\mathbb P}_{\rho}.$
Then, for all $q\in [1,+\infty],$ there exist constants $c,c'>0$ such that the following bounds hold:
\begin{equation}
 \label{minmaxthm1boundq_U_P}
 \underset{\hat{\rho}}{\inf}\underset{\rho\in\mathcal{S}_{r,m}}{\sup}\mathbb{P}_{\rho}
 \biggl\{\|\hat{\rho}-\rho\|_q\geq c\biggl(\frac{m r^{1/q}}{\sqrt{n}}\bigwedge \biggl(\frac{m}{\sqrt{n}}\biggr)^{1-\frac{1}{q}}\bigwedge 1\biggr)\biggr\}\geq c',
\end{equation}
 \begin{equation}
 \label{minmaxthm1boundH_U_P}
 \underset{\hat{\rho}}{\inf}\underset{\rho\in\mathcal{S}_{r,m}}{\sup}\mathbb{P}_{\rho}\biggl\{H^2(\hat{\rho},\rho)\geq c\biggl(\frac{m r}{\sqrt{n}}\bigwedge 1\biggr)\biggr\}\geq c',
 \end{equation}
and
 \begin{equation}
 \label{minmaxthm1boundK_U_P}
 \underset{\hat{\rho}}{\inf}\underset{\rho\in\mathcal{S}_{r,m}}{\sup}\mathbb{P}_{\rho}\biggl\{K(\rho\|\hat{\rho})\geq c\biggl(\frac{m r}{\sqrt{n}}\bigwedge 1\biggr)\biggr\}\geq c',
 \end{equation}
 where $\inf_{\hat{\rho}}$ denotes the infimum over all estimators $\hat{\rho}$ in $\mathcal{S}_m$
 based on the data $(X_1,Y_1), \dots, (X_n,Y_n).$ 
\end{theorem}

\begin{remark}
\label{Gross_Flammia}
Minimax lower bounds on nuclear norm error of density matrix estimation close to bound (\ref{minmaxthm1boundq_U_P}) for $q=1$ 
(but for a somewhat different ``estimation protocol" and stated in a different form)
were obtained earlier in \cite{flammia2012quantum}. This paper also contains upper 
bounds on the errors of matrix LASSO and Dantzig selector estimators in the nuclear norm
matching the lower bounds up to log-factors. 
\end{remark}

\begin{remark}
It is easy to see that, if constant $\gamma\in (0,1)$ is small enough (namely, $\gamma<1-\frac{1}{\sqrt{2}}$), then, in an arbitrary orthonormal basis $\{E_1,\dots, E_{m^2}\},$ there is {\it at most one} matrix $E_j$ such that $|{\rm tr}(E_j)|>(1-\gamma)Um.$ Indeed, note that 
${\rm tr}(E_j)=\langle E_j,I_m\rangle.$ Since 
$$
\sum_{j=1}^{m^2} \langle E_j,I_m\rangle^2 = \|I_m\|_2^2 =m
$$
and $U^2m \geq 1,$ we have 
$$
{\rm card}\Bigl(\Bigl\{j: |\langle E_j,I_m\rangle|>(1-\gamma)Um\Bigr\}\Bigr)\leq \frac{1}{(1-\gamma)^2 U^2 m^2}\sum_{j=1}^{m^2} \langle E_j,I_m\rangle^2
$$
$$
\leq \frac{m}{(1-\gamma)^2 U^2 m^2}=\frac{1}{(1-\gamma)^2U^2m}\leq \frac{1}{(1-\gamma)^2}<2,
$$
provided that $\gamma<1-\frac{1}{\sqrt{2}}.$ 
\end{remark}

\begin{remark}
It will be shown in Section \ref{tildesec} that the minimax rates of theorems \ref{minmaxthm1}, \ref{minmaxthm3}, \ref{minmaxthm3'''} and \ref{minmaxthm3_Pauli} are attained up to logarithmic factors for the von Neumann entropy penalized least squares estimator. 
\end{remark}

\begin{remark}
Similar minimax lower bounds could be proved in certain classes of  ``nearly low rank" density matrices. Consider,
for instance, the following class  
\begin{equation}
\label{Bpddef}
 B_p(d;m):=\biggl\{S\in \mathcal{S}_m: \sum_{j=1}^m|\lambda_j(S)|^p\leq d\biggr\}
\end{equation}
for some $d>0$ and $p\in [0,1],$ where $\lambda_1(S)\geq \dots\geq \lambda_m(S)$ denote 
the eigenvalues of $S.$  This set consists of density matrices with the 
eigenvalues decaying at a certain rate (nearly low rank case) and, for $p=0,$ $d=r$ it coincides with 
${\mathcal S}_{r,m}.$ 
It turns out that minimax lower bounds of theorems \ref{minmaxthm1} and \ref{minmaxthm3} hold for 
the class $B_p(d;m)$ (instead of ${\mathcal S}_{r,m}$) with $r$ replaced by
$$
\bar r := \bar r (\tau, d,m,p)= d\tau^{-p}\wedge m,
$$
where
$\tau:= \frac{\sigma_{\xi}m^{3/2}}{\sqrt{n}}$ in the case of trace regression with Gaussian noise and 
$\tau:= \frac{\bar U m^{3/2}}{\sqrt{n}}$ in the case of trace regression with bounded response. 
These minimax bounds are attained up to logarithmic factors for a slightly modified von Neumann entropy 
penalized least squares estimator.  

Note that, for $\rho\in B_p(d,m)$ with eigenvalues $\lambda_1(\rho)\geq \dots \geq \lambda_m(\rho),$
we have $\lambda_j(\rho)\leq \frac{d^{1/p}}{j^{1/p}}, j=1,\dots, m.$ Therefore, for $j\geq \bar r,$
$\lambda_j(\rho)\leq \tau.$ Note also that $\tau$ characterizes the minimax rate of estimation 
of $\rho\in {\mathcal S}_{r,m}$ in the operator norm for any value of the rank $r$ (see 
bound (\ref{minmaxthm1boundq}) for $q=+\infty;$ the corresponding upper bound also 
holds for the least squares estimator up to a logarithmic factor, see \citealt{KoltchinskiiXia2015}). 
Roughly speaking, $\tau$ is a threshold 
below which the estimation of eigenvalues $\lambda_j(\rho)$ becomes impossible and $\bar r$
can be viewed as an ``effective rank" of nearly low rank density matrices in the class $B_p(d,m).$
\end{remark}

\section{Von Neumann Entropy Penalization: Optimality and Oracle Inequalities}
\label{vonneumann_0}

The goal of this section is to study optimality properties of von Neumann entropy penalized least squares estimator $\tilde \rho^{\eps}$ defined by (\ref{trwvne}). In particular, we establish oracle inequalities for 
such estimators in the cases of trace regression with bounded response (Subsection \ref{vonneumann_1})
and trace regression with Gaussian noise (Subsection \ref{vonneumann_2}), and prove upper bounds on their estimation errors measured by Schatten $q$-norm distances 
for $q\in [1,2]$ and also by Hellinger and Kullback-Leibler distances (Subsection \ref{tildesec}). 

\subsection{Oracle Inequalities for Trace Regression with Bounded Response}
\label{vonneumann_1}

In this subsection, we prove a {\it sharp low rank oracle inequality} for estimator $\tilde \rho^{\eps}$ defined by (\ref{trwvne}). It is done in the case of trace regression model with bounded response (that is, under Assumption \ref{bounded_response}). 
The results of this type show some form of optimality of the estimation method, namely, that the estimator provides an optimal trade-off between the ``approximation error" of the target density matrix by a low rank ``oracle" and the ``estimation error" of the ``oracle'' 
that is proportional to its rank. Sharp oracle inequalities (in which the leading constant in front of the ``approximation error"  is equal to $1,$ so that the bound mimics precisely the approximation by the oracle) are usually harder to prove. In the case of low rank matrix completion, the first result of this type was proved by \cite{koltchinskii2011nuclear} for a modified least 
squares estimator with nuclear norm penalty. A version of such inequality for empirical risk minimization with nuclear norm penalty
(that includes matrix LASSO) was first proved by \cite{koltchinskii2013sharp}. Low rank oracle inequalities for von Neumann 
entropy penalized least squares method with the leading constant larger than $1$ were proved by \cite{koltchinskii2011neumann}. The main result of this section refines these previous bounds by proving a sharp oracle 
inequality, improving the logarithmic factors and removing superfluous assumptions, but also by establishing the inequality in the whole range of values of regularization parameter $\eps\geq 0$ (including the value $\eps=0,$
for which $\tilde \rho^{\eps}$ coincides with the least squares estimator $\hat \rho$). In addition to this, for a special choice of regularization parameter $\eps,$
the theorem below also provides an upper bound on the Kullback-Leibler error $K(\rho\|\tilde \rho^{\eps})$ of $\tilde \rho^{\eps}$ 
that matches the minimax lower bound (\ref{minmaxthm1boundK_U}) up to log-factors
(and ``second order terms"). 
It turns out that, for this choice of $\eps,$ the estimator satisfies exactly the same low rank oracle inequality as the best inequalities known for LASSO estimator and minimax optimal 
error rates are attained for $\tilde \rho^{\eps}$ also with respect to Hellinger distance and Schatten 
$q$-norm distances for all $q\in [1,2]$ (see Section \ref{tildesec}).  
For simplicity, it will be assumed that constants 
$U$ in Assumption \ref{orthonorm} and $\bar U$ in Assumption \ref{bounded_response} coincide (in the upper bounds, one 
can always replace $U$ and $\bar U$ by $U\vee \bar U$).

\begin{theorem} 
\label{th-KL-1}
Suppose Assumption \ref{bounded_response} holds with constant $\bar U=U$ and 
let $\eps\in [0,1].$
Then, there exists a constant $C>0$ such that 
for all $t\geq 1$ with probability at least $1-e^{-t}$
\begin{eqnarray}
\label{sharp_oracle}
&
\nonumber
\|f_{\tilde \rho^{\eps}}-f_{\rho}\|_{L_2(\Pi)}^2
\leq 
\inf_{S\in {\mathcal S}_m}\biggl[\|f_{S}-f_{\rho}\|_{L_2(\Pi)}^2
+
C\biggl({\rm rank}(S) m^2 \eps^2 \log^2 (mn)
\\
&
+  U^2\frac{{\rm rank}(S)m \log(2m)}{n}+U^2\frac{t+\log \log_2 (2n)}{n}\biggr)
\biggr].
\end{eqnarray}
In particular, this implies that 
\begin{eqnarray}
\label{L_2-error}
&
\nonumber
\|f_{\tilde \rho^{\eps}}-f_{\rho}\|_{L_2(\Pi)}^2
\leq 
C\biggl[{\rm rank}(\rho) m^2 \eps^2 \log^2 (mn)
\\
&
+  U^2\frac{{\rm rank}(\rho)m \log(2m)}{n}
+U^2\frac{t+\log \log_2 (2n)}{n}\biggr].
\end{eqnarray}
Moreover, if 
$$
\eps:=\frac{1}{\log(mn)}\biggl[U\sqrt{\frac{\log(2m)}{n m}}\bigvee U^2\frac{\log (2m)}{n}\biggr],
$$
then, with some constant $C$ and with probability at least $1-e^{-t}$
\begin{eqnarray}
\label{L_2-error-A}
&
\nonumber
\|f_{\tilde \rho^{\eps}}-f_{\rho}\|_{L_2(\Pi)}^2
\leq 
C\biggl[U^2\frac{{\rm rank}(\rho)m \log(2m)}{n}\biggl(1\bigvee U^2\frac{m\log(2m)}{n}\biggr)
\\
&
+U^2\frac{t+\log \log_2 (2n)}{n}\biggr]
\end{eqnarray}
and
\begin{eqnarray}
\label{KL-error}
&
\nonumber
K(\rho\|\tilde \rho^{\eps})
\leq 
CU\biggl[\frac{{\rm rank}(\rho)m^{3/2} \sqrt{\log(2m)}\log (mn)}{\sqrt{n}}
\biggl(1\bigvee U\sqrt{\frac{m\log(2m)}{n}}\biggr)
\\
&
+\sqrt{\frac{m}{n}}\frac{(t+\log \log_2 (2n))\log(mn)}{\sqrt{\log (2m)}}
\biggr].
\end{eqnarray}
\end{theorem}

\begin{proof}
The following notations will be used in the proof. 
Let $\ell(y,u):=(u-y)^2, y,u\in {\mathbb R}$ be the quadratic loss function. For $f:{\mathbb H}_m\mapsto {\mathbb R},$ denote 
$$
(\ell \bullet f)(x,y)=(f(x)-y)^2, \ \ (\ell^{\prime}\bullet f)(x,y)=2(f(x)-y)
$$
and 
$$
P(\ell \bullet f)={\mathbb E}(Y-f(X))^2,\ \ P_n (\ell \bullet f)=n^{-1}\sum_{j=1}^n (Y_j-f(X_j))^2.  
$$
For $A\in {\mathbb H}_m,$ let $f_A(x)=\langle A,x\rangle, x\in {\mathbb H}_m.$ 
Since for density matrices $S\in {\mathcal S}_m,$ $\|S\|_1={\rm tr}(S)=1,$
the estimator $\tilde \rho=\tilde \rho^{\eps}$ can be equivalently defined by the following convex optimization problem:
$$
\tilde \rho = {\rm argmin}_{S\in {\mathcal S}_m} L_n(S),\ \ L_n(S):=
\Bigl[P_n (\ell \bullet f_S)+\eps {\rm tr}(S\log S)+
\bar \eps \|S\|_1\Bigr]
$$
for an arbitrary $\bar \eps>0.$ 

The following lemma will be crucial in the proofs of Theorem \ref{th-KL-1} as well Theorem \ref{th-KL-2} in the following 
subsection. 
Note that it does not rely on 
Assumption  \ref{bounded_response}, only Assumptions \ref{orthonorm} and 
\ref{unif_sample} are needed. 

\begin{lemma}
\label{lemma_intermed}
Suppose Assumptions \ref{orthonorm} and 
\ref{unif_sample} hold. 
Let $\delta\in (0,1)$ and $S:=(1-\delta )S'+ \delta\frac{I_m}{m},$ where $S'\in {\mathcal S}_m,$
${\rm rank}(S')=r$ and $I_m$ is the $m\times m$ identity matrix. Then the following bound
holds:
\begin{eqnarray}
\label{intermed_2}
&
\nonumber
\|f_{\tilde \rho}-f_{\rho}\|_{L_2(\Pi)}^2+\frac{1}{2}\|f_{\tilde \rho}-f_S\|_{L_2(\Pi)}^2
+ \eps K(\tilde \rho;S)
+\bar \eps \Bigl\|{\mathcal P}_L^{\perp}(\tilde \rho )\Bigr\|_1
\\
&
\leq 
\|f_S-f_{\rho}\|_{L_2(\Pi)}^2+ r m^2 \eps^2 \log^2 (m/\delta)
+ r m^2 \bar \eps^2 
\\
&
\nonumber
+4 \bar\eps\delta+ 
(P-P_n)(\ell^{\prime}\bullet f_{\tilde \rho})(f_{\tilde \rho}-f_S).
\end{eqnarray}
\end{lemma}

Lemma \ref{lemma_intermed} will be often used together with the following 
simple bound:
\begin{eqnarray}
\label{S'S}
&
\nonumber
\|f_S-f_{\rho}\|_{L_2(\Pi)}^2 = \frac{1}{m^2}\|S-\rho\|_2^2 
\leq
\\
& 
\frac{1}{m^2}\|S^{\prime}-\rho\|_2^2 +
\frac{2}{m^2}\|S^{\prime}-\rho\|_2\|S^{\prime}-S\|_2 + 
\frac{1}{m^2}\|S^{\prime}-S\|_2^2 
\\
&
\nonumber
\leq 
\|f_{S^{\prime}}-f_{\rho}\|_{L_2(\Pi)}^2 +
\frac{8\delta}{m^2} + 
\frac{4\delta^2}{m^2}
\leq 
\|f_{S^{\prime}}-f_{\rho}\|_{L_2(\Pi)}^2 +
\frac{12\delta}{m^2}. 
\end{eqnarray}
Together, they imply that 
\begin{eqnarray}
\label{intermed_2''}
&
\nonumber
\|f_{\tilde \rho}-f_{\rho}\|_{L_2(\Pi)}^2+\frac{1}{2}\|f_{\tilde \rho}-f_S\|_{L_2(\Pi)}^2
+ \eps K(\tilde \rho;S)
+\bar \eps \Bigl\|{\mathcal P}_L^{\perp}(\tilde \rho )\Bigr\|_1
\\
&
\leq 
\|f_{S'}-f_{\rho}\|_{L_2(\Pi)}^2+ r m^2 \eps^2 \log^2 (m/\delta)
+ r m^2 \bar \eps^2 
\\
&
\nonumber
+4 \bar \eps\delta+ \frac{12\delta}{m^2}+
(P-P_n)(\ell^{\prime}\bullet f_{\tilde \rho})(f_{\tilde \rho}-f_S).
\end{eqnarray}

We will now give the proof of Lemma \ref{lemma_intermed}. 

\begin{proof}
By standard necessary conditions of extremum in convex problems, we get that, for all $S\in {\mathcal S}_m$ and for some $\tilde V\in \partial \|\tilde \rho\|_1,$  
$$
P_n(\ell^{\prime}\bullet f_{\tilde \rho})(f_{\tilde \rho}-f_S) + \eps \langle \log \tilde \rho, \tilde \rho-S\rangle
+\bar \eps \langle \tilde V, \tilde \rho-S\rangle\leq 0
$$
(see, e.g., \citealt{aubin}, Chapter 2, Corollary 6; see also 
\citealt{koltchinskii2011oracle}, pp. 198--199; for the computation of derivative of the function ${\rm tr}(S\log S),$ see Lemma 1 in \citealt{koltchinskii2011neumann}). 
Replacing in the left hand side $P$ by $P_n,$ we get 
$$
P(\ell^{\prime}\bullet f_{\tilde \rho})(f_{\tilde \rho}-f_S) + \eps \langle \log \tilde \rho, \tilde \rho-S\rangle
+\bar \eps \langle \tilde V, \tilde \rho-S\rangle\leq 
(P-P_n)(\ell^{\prime}\bullet f_{\tilde \rho})(f_{\tilde \rho}-f_S).
$$
It is easy to check that for the quadratic loss 
$$
P(\ell^{\prime}\bullet f_{\tilde \rho})(f_{\tilde \rho}-f_S)= 
P(\ell\bullet f_{\tilde \rho})-P(\ell\bullet f_S)
+\|f_{\tilde \rho}-f_S\|_{L_2(\Pi)}^2, 
$$
implying that 
$$
P(\ell\bullet f_{\tilde \rho})-P(\ell\bullet f_S)
+\|f_{\tilde \rho}-f_S\|_{L_2(\Pi)}^2
+ \eps \langle \log \tilde \rho, \tilde \rho-S\rangle
+\bar \eps \langle \tilde V, \tilde \rho-S\rangle
$$
$$
\leq 
(P-P_n)(\ell^{\prime}\bullet f_{\tilde \rho})(f_{\tilde \rho}-f_S).
$$
Also, for the quadratic loss,
$$
P(\ell \bullet f)-P(\ell \bullet f_{\rho})= \|f-f_{\rho}\|_{L_2(\Pi)}^2.
$$
Therefore, 
$$
\|f_{\tilde \rho}-f_{\rho}\|_{L_2(\Pi)}^2+\|f_{\tilde \rho}-f_S\|_{L_2(\Pi)}^2
+ \eps \langle \log \tilde \rho, \tilde \rho-S\rangle
+\bar \eps \langle \tilde V, \tilde \rho-S\rangle
$$
$$
\leq 
\|f_S-f_{\rho}\|_{L_2(\Pi)}^2+(P-P_n)(\ell^{\prime}\bullet f_{\tilde \rho})(f_{\tilde \rho}-f_S).
$$
Recall that we have set $S=(1-\delta)S^{\prime}+\delta \frac{I_m}{m},$ where 
$S^{\prime}\in {\mathcal S}_m,$ ${\rm rank}(S^{\prime})=r,$ $\delta\in (0,1).$ Clearly, 
$$
\Bigl|\langle \tilde V, S-S^{\prime}\rangle\Bigr|\leq \|\tilde V\|_{\infty}\|S-S^{\prime}\|_1
\leq  \|S-S^{\prime}\|_1=\delta \biggl\|S^{\prime}-\frac{I_m}{m}\biggr\|_1\leq 2\delta,
$$
where we used the fact that $\|\tilde V\|_{\infty}\leq 1$ for $\tilde V \in \partial \|\tilde \rho\|_1.$
This implies 
\begin{eqnarray}
\label{intermed}
&
\|f_{\tilde \rho}-f_{\rho}\|_{L_2(\Pi)}^2+\|f_{\tilde \rho}-f_S\|_{L_2(\Pi)}^2
+ \eps \langle \log \tilde \rho, \tilde \rho-S\rangle
+\bar \eps \langle \tilde V, \tilde \rho-S^{\prime}\rangle
\\
&
\nonumber
\leq 
\|f_S-f_{\rho}\|_{L_2(\Pi)}^2+2 \bar \eps\delta+ 
(P-P_n)(\ell^{\prime}\bullet f_{\tilde \rho})(f_{\tilde \rho}-f_S).
\end{eqnarray}

Recall formula (\ref{subdiff}) for the subdifferential  of nuclear norm. 
Let $L={\rm supp}(S^{\prime}).$
By the duality between the operator and nuclear norms, there exists $M\in {\mathbb H}_m$
with $\|M\|_{\infty}\leq 1$ such that 
$$
\langle {\mathcal P}_L^{\perp}(M),\tilde \rho - S^{\prime}\rangle
=  
\langle M,{\mathcal P}_L^{\perp}(\tilde \rho - S^{\prime})\rangle
=\Bigl\|{\mathcal P}_L^{\perp}(\tilde \rho - S^{\prime})\Bigr\|_1=
\Bigl\|{\mathcal P}_L^{\perp}(\tilde \rho )\Bigr\|_1.
$$
With $V={\rm sign}(S^{\prime})+{\mathcal P}_L^{\perp}(M)\in \partial \|S^{\prime}\|_1,$
by monotonicity of subdifferential, we get that 
\begin{equation}
\label{subdf}
\langle {\rm sign}(S^{\prime}), \tilde \rho-S^{\prime}\rangle + 
\Bigl\|{\mathcal P}_L^{\perp}(\tilde \rho )\Bigr\|_1=
\langle V, \tilde \rho-S^{\prime}\rangle\leq \langle \tilde V, \tilde \rho-S^{\prime}\rangle.
\end{equation}
In addition to this, we have 
\begin{equation}
\label{subdif_VN}
\langle \log \tilde \rho, \tilde \rho-S\rangle= 
\langle \log \tilde \rho-\log S, \tilde \rho-S\rangle
+\langle \log S, \tilde \rho-S\rangle
= K(\tilde \rho; S)+ 
\langle \log S, \tilde \rho-S\rangle.
\end{equation}
Substituting (\ref{subdf}) and (\ref{subdif_VN}) into (\ref{intermed}), we get 
\begin{eqnarray}
\label{intermed_1}
&
\nonumber
\|f_{\tilde \rho}-f_{\rho}\|_{L_2(\Pi)}^2+\|f_{\tilde \rho}-f_S\|_{L_2(\Pi)}^2
+ \eps K(\tilde \rho;S)
+\bar \eps \Bigl\|{\mathcal P}_L^{\perp}(\tilde \rho )\Bigr\|_1
\\
&
\leq 
\|f_S-f_{\rho}\|_{L_2(\Pi)}^2+ \eps \langle \log S, S-\tilde \rho\rangle + \bar \eps 
\langle {\rm sign}(S^{\prime}), S^{\prime}-\tilde \rho\rangle
\\
&
\nonumber
+2 \bar \eps\delta+ 
(P-P_n)(\ell^{\prime}\bullet f_{\tilde \rho})(f_{\tilde \rho}-f_S).
\end{eqnarray}
The following bound on $\bar \eps\langle {\rm sign}(S^{\prime}), S^{\prime}-\tilde \rho\rangle$ is straightforward:
\begin{eqnarray}
\label{nuce}
&
\nonumber
\bar \eps \langle {\rm sign}(S^{\prime}), S^{\prime}-\tilde \rho\rangle\leq 
\bar \eps \langle {\rm sign}(S^{\prime}), S-\tilde \rho\rangle+ \bar \eps \|{\rm sign}(S')\|_{\infty}\|S-S^{\prime}\|_{1}
\\
&
\leq \bar \eps\|{\rm sign}(S^{\prime})\|_2\|S-\tilde \rho\|_2+ 2\bar \eps \delta\leq 
\bar \eps \sqrt{r}m \|f_S-f_{\tilde \rho}\|_{L_2(\Pi)}+ 2\bar \eps \delta
\\
&
\nonumber
\leq r m^2 \bar \eps^2 + \frac{1}{4}\|f_S-f_{\tilde \rho}\|_{L_2(\Pi)}^2 + 2\bar \eps \delta.
\end{eqnarray}

A similar bound on $\eps \langle \log S, S-\tilde \rho\rangle$ is only slightly 
more complicated. Suppose $S^{\prime}$ has the following spectral representation:
$S^{\prime}=\sum_{k=1}^r \lambda_k P_k$ with eigenvalues $\lambda_k\in (0,1]$ 
(repeated with their multiplicities) and one-dimensional 
orthogonal eigenprojectors $P_k.$ We will extend $P_j, j=1,\dots, r$ to the complete orthogonal resolution of the identity $P_j, j=1,\dots, m.$ 
Then 
$$
\log S = \log \biggl((1-\delta)S^{\prime}+\delta \frac{I_m}{m}\biggr)=
\sum_{j=1}^r \log \Bigl((1-\delta)\lambda_j+\delta/m\Bigr)P_j+ \sum_{j=r+1}^m\log(\delta/m)P_j
$$
$$
=\sum_{j=1}^r \log \Bigl(1+(1-\delta)m\lambda_j/\delta\Bigr)P_j + \log(\delta/m)I_m
$$ 
and 
$$
\langle \log S, S-\tilde \rho\rangle = 
\biggl\langle \sum_{j=1}^r \log \Bigl(1+(1-\delta)m\lambda_j/\delta\Bigr)P_j,
S-\tilde \rho\biggr\rangle 
+ \log(\delta/m) \langle I_m, S-\tilde \rho\rangle
$$
$$
= 
\biggl\langle \sum_{j=1}^r \log \Bigl(1+(1-\delta)m\lambda_j/\delta\Bigr)P_j,
S-\tilde \rho\biggr\rangle$$
where we used the fact that $\langle I_m, S-\tilde \rho\rangle= {\rm tr}(S)-{\rm tr}(\tilde \rho)=0.$
Therefore,
\begin{eqnarray}
\label{VNeu}
&
\eps \langle \log S, S-\tilde \rho\rangle
\leq \eps \biggl\|\sum_{j=1}^r \log \Bigl(1+(1-\delta)m\lambda_j/\delta\Bigr)P_j\biggr\|_2 
\|S-\tilde \rho\|_2 
\\
&
\nonumber
= \eps m \biggl(\sum_{j=1}^r \log^2 \Bigl(1+(1-\delta)m\lambda_j/\delta\Bigr)\biggr)^{1/2}\|f_S-f_{\tilde \rho}\|_{L_2(\Pi)}
\\
&
\nonumber
\leq \eps \sqrt{r} m \log (m/\delta)\|f_S-f_{\tilde \rho}\|_{L_2(\Pi)} 
\leq 
r m^2 \eps^2 \log^2 (m/\delta)+ \frac{1}{4}\|f_S-f_{\tilde \rho}\|_{L_2(\Pi)}^2,
\end{eqnarray}
where it was used that for $\lambda_j\in [0,1]$
$$
\log \Bigl(1+(1-\delta)m\lambda_j/\delta\Bigr)\leq 
\log \Bigl(\frac{\delta +(1-\delta)m}{\delta}\Bigr)\leq \log(m/\delta).
$$
Substituting bounds (\ref{nuce}) and (\ref{VNeu}) in (\ref{intermed_1}) we easily get
bound (\ref{intermed_2}), as claimed in the lemma.


\end{proof}

We will also need the following simple lemma that provides a bound on 
$K(S^{\prime}\|\tilde \rho)$ in terms of  $K(S\|\tilde \rho).$

Let 
$$h(\delta):=\delta \log \frac{1}{\delta}+(1-\delta)\log \frac{1}{1-\delta}.$$
Observe that 
$$
h(\delta)= \delta \log\frac{1}{\delta}+ (1-\delta)\log \biggl(1+\frac{\delta}{1-\delta}\biggr)
\leq \delta \log\frac{1}{\delta}+(1-\delta)\frac{\delta}{1-\delta}\leq \delta \log \frac{e}{\delta}
$$
(this bound will be used in what follows).

\begin{lemma}
\label{lem-kull}
Let $\delta\in (0,1),$ $S^{\prime}\in {\mathcal S}_m$ with ${\rm rank}(S')=r$ and $S=(1-\delta)S^{\prime}+\delta \frac{I_m}{m}.$
Then, for any $U\in {\mathcal S}_m,$
$$
K(S^{\prime}\|U)\leq \frac{K(S\|U)+h(\delta)}{1-\delta}.
$$
\end{lemma}

\begin{proof}
The following identities are straightforward:
\begin{eqnarray*}
&
K(S\|U)= {\rm tr}(S(\log S-\log U))
\\
&
=(1-\delta){\rm tr}(S^{\prime}(\log S-\log U))+
\delta {\rm tr}((I_m/m)(\log S-\log U))
\\
&
=(1-\delta){\rm tr}(S^{\prime}(\log S^{\prime}-\log U))
+(1-\delta){\rm tr}(S^{\prime}(\log S-\log S^{\prime}))
\\
&
+
\delta {\rm tr}((I_m/m)(\log S-\log (I_m/m)))
+
\delta {\rm tr}((I_m/m)(\log (I_m/m)-\log U))
\\
&
=(1-\delta)K(S^{\prime}\|U)-(1-\delta)K(S^{\prime}\|S)+
\delta K(I_m/m\|U)-\delta K(I_m/m\|S).
\end{eqnarray*}
Since $K(I_m/m\|U)\geq 0,$ it follows that 
\begin{equation}
\label{KL-od}
K(S^{\prime}\|U)\leq \frac{K(S\|U)}{1-\delta}+ K(S^{\prime}\|S)+\frac{\delta}{1-\delta}K(I_m/m\|S).
\end{equation}
Assuming that $S^{\prime}$ has spectral representation $S^{\prime}=\sum_{j=1}^r \lambda_j P_j$
with eigenvalues $\lambda_j>0$ and one-dimensional projectors $P_j,$ we get 
$$
-K(S^{\prime}\|S)= \sum_{j=1}^r \lambda_j \log \frac{(1-\delta)\lambda_j+\delta/m}{\lambda_j}
$$
$$
=\sum_{j=1}^r \lambda_j \log\biggl(1-\delta+\frac{\delta}{m\lambda_j}\biggr)
\geq \log (1-\delta)\sum_{j=1}^r \lambda_j=\log(1-\delta),
$$
implying that $K(S^{\prime}\|S)\leq \log\frac{1}{1-\delta}.$ On the other hand,
$$
K(I_m/m\|S)=\frac{1}{m}\sum_{j=1}^m \log \frac{1/m}{(1-\delta)\lambda_j+\delta/m}
\leq \frac{1}{m}\sum_{j=1}^m \log\frac{1}{\delta}=\log\frac{1}{\delta}.
$$
Substituting these bounds in (\ref{KL-od}) yields the result.
\end{proof}

To complete the proof of Theorem \ref{th-KL-1}, we need to control the empirical process 
$(P-P_n)(\ell'\bullet f_{\tilde \rho})(f_{\tilde \rho}-f_S)$ in the right hand side of bound 
(\ref{intermed_2}). Our approach is based on the following empirical processes bound that is a slight modification 
of Lemma 1 in \cite{koltchinskii2013sharp}. As before, we assume that 
$S=(1-\delta)S'+\delta \frac{I_m}{m}$ with $S'\in {\mathcal S}_m,$ ${\rm rank}(S')=r.$
We will set $\delta:=\frac{1}{m^2 n^2}.$

Let $\Xi_{\eps}:=n^{-1}\sum_{j=1}^n \eps_j X_j,$ where ${\eps_j}$ are i.i.d. Rademacher 
random variables (that is, $\eps_j$ takes values $+1$ and $-1$ with probability $1/2$ each)
and $\{\eps_j\}, \{X_j\}$ are independent. 

\begin{lemma}
\label{lemma_alpha_n}
Given $\delta_1, \delta_2>0,$ denote 
$$
\alpha_n(\delta_1,\delta_2):= \sup\biggl\{\Bigl|(P_n-P)(\ell'\bullet f_A)(f_A-f_S)\Bigr|: A\in {\mathcal S}_m,
\|f_A-f_S\|_{L_2(\Pi)}\leq \delta_1, \|{\mathcal P}_L^{\perp}A\|_1\leq \delta_2
 \biggr\}.
$$
Let $0<\delta_1^{-}<\delta_1^{+}, 0<\delta_2^{-}<\delta_2^{+}.$ For $t\geq 1,$ denote 
$$
\bar t := t+ 
\log\Bigl([\log_2(\delta_1^{+}/\delta_1^{-})]+2\Bigr)+
\log\Bigl([\log_2(\delta_2^{+}/\delta_2^{-})]+2\Bigr)+
\log 3.
$$
Then, with probability at least $1-e^{-t},$ for all $\delta_1\in [\delta_1^-,\delta_1^+],
\delta_2\in [\delta_2^-,\delta_2^+],$
$$
\alpha_n(\delta_1,\delta_2)\leq C_1 U {\mathbb E}\|\Xi_{\eps}\|_{\infty} \Bigl(\sqrt{r}m \delta_1+\delta_2+\delta\Bigr)
+ C_2 U\delta_1\sqrt{\frac{\bar t}{n}}+ C_3 U^2 \frac{\bar t}{n},
$$
where $C_1, C_2,C_3>0$ are constants.
\end{lemma}

We will use this lemma to control the term $(P-P_n)(\ell'\bullet f_{\tilde \rho})(f_{\tilde \rho}-f_S)$ in bound 
(\ref{intermed_2}). Let $\delta_1:=\|f_{\tilde \rho}-f_S\|_{L_2(\Pi)}$ and 
$\delta_2:=\|{\mathcal P}_L^{\perp}\tilde \rho\|_1.$ Define also 
$$
\delta_1^{+}:=\frac{2}{m},\  \delta_2^{+}:=1,\ \delta_1^{-}=\delta_2^{-}:= \frac{1}{m n},
$$
so that $\bar t\leq t+2\log (\log_2 (mn)+3)+\log 3.$ 
It is easy to see that $\delta_1\leq \delta_1^{+}$ and $\delta_2\leq \delta_2^{+}.$
If, in addition, $\delta_1\geq \delta_1^-,$ $\delta_2\geq \delta_2^{-},$
the bound of Lemma \ref{lemma_alpha_n} implies that with probability
at least $1-e^{-t}$
$$
(P-P_n)(\ell'\bullet f_{\tilde \rho})(f_{\tilde \rho}-f_S)\leq 
\alpha_n(\delta_1,\delta_2)
$$
$$
\leq 
C_1 U {\mathbb E}\|\Xi_{\eps}\|_{\infty} \Bigl(\sqrt{r}m \delta_1+\delta_2+\delta\Bigr)
+ C_2 U\delta_1\sqrt{\frac{\bar t}{n}}+ C_3 U^2 \frac{\bar t}{n}
$$
If $\bar \eps \geq C_1 U {\mathbb E}\|\Xi_{\eps}\|_{\infty},$ the last bound implies that 
\begin{eqnarray}
\label{empbd}
&
\nonumber
(P-P_n)(\ell'\bullet f_{\tilde \rho})(f_{\tilde \rho}-f_S)
\\
&
\leq \frac{1}{4}\|f_{\tilde \rho}-f_S\|_{L_2(\Pi)}^2
+ rm^2 \bar \eps^2 + 
\bar \eps \|{\mathcal P}_L^{\perp}\tilde \rho\|_1+\bar \eps \delta
\\
&
\nonumber
+\frac{1}{4}\|f_{\tilde \rho}-f_S\|_{L_2(\Pi)}^2 + (C_2^2+C_3) U^2\frac{\bar t}{n}.
\end{eqnarray}
Substituting this bound in the right hand side of (\ref{intermed_2''}), we get 
\begin{eqnarray}
\label{intermed_11}
&
\nonumber
\|f_{\tilde \rho}-f_{\rho}\|_{L_2(\Pi)}^2
+ \eps K(\tilde \rho;S)
\\
&
\leq 
\|f_{S'}-f_{\rho}\|_{L_2(\Pi)}^2+ r m^2 \eps^2 \log^2 (m/\delta)
+  2r m^2 \bar \eps^2 
\\
&
\nonumber
+5\bar \eps\delta+ CU^2\frac{\bar t}{n}+\frac{12\delta}{m^2},
\end{eqnarray}
where $C:=C_2^2+C_3.$

In the case when $\delta_1= \|f_{\tilde \rho}-f_S\|_{L_2(\Pi)}\leq \delta_1^{-}=\frac{1}{mn}$ or
$\delta_2=\|{\mathcal P}_L^{\perp}\tilde \rho\|_1\leq \delta_2^{-}=\frac{1}{mn},$ we can replace 
the terms $\frac{1}{4}\|f_{\tilde \rho}-f_S\|_{L_2(\Pi)}^2$ or 
$\|{\mathcal P}_L^{\perp}\tilde \rho\|_1$ in bound (\ref{empbd}) by their respective upper bounds 
($\frac{1}{4}(\delta_1^-)^2=\frac{1}{4 m^2n^2},$
or $\delta_2^{-}=\frac{1}{mn}$), which would be smaller than  $CU^2\frac{\bar t}{n}$ for large enough 
$C>0,$ so bound (\ref{intermed_11}) still holds (recall that $U\geq m^{-1/2}$). 
Note also that
$\frac{12\delta}{m^2}= 12\frac{1}{m^4 n^2} \leq 12U^2\frac{\bar t}{n}.$
Thus, increasing the value of constant $C,$ one can rewrite (\ref{intermed_11})
in a simpler form as 
\begin{eqnarray}
\label{intermed_12}
&
\nonumber
\|f_{\tilde \rho}-f_{\rho}\|_{L_2(\Pi)}^2
+ \eps K(\tilde \rho;S)
\\
&
\leq 
\|f_{S'}-f_{\rho}\|_{L_2(\Pi)}^2+ r m^2 \eps^2 \log^2 (m/\delta)
+  2 r m^2 \bar \eps^2 
\\
&
\nonumber
+5\bar \eps\delta+ CU^2\frac{\bar t}{n}.
\end{eqnarray}
The following expectation bound is a consequence of a matrix version of Bernstein inequality
for $\|\Xi_{\eps}\|_{\infty}$ (it follows by integrating out its exponential tails):
$$
{\mathbb E}\|\Xi_{\eps}\|_{\infty}\leq 
4\biggl[\sqrt{\frac{\log(2m)}{n m}}\bigvee U\frac{\log (2m)}{n}\biggr]
$$
(it is also used in this computation that, in the case of uniform sampling from an orthonormal basis, 
$\sigma_{\eps X}^2=\|{\mathbb E}X^2\|_{\infty}=\frac{1}{m},$
a simple fact often used in the literature;
see, e.g., \citealt{koltchinskii2011neumann}, Section 5). 
Let 
$$
\bar \eps := D'U\sqrt{\frac{\log(2m)}{n m}}
$$
for some constant $D'.$ If $D'$ is sufficiently large
and 
\begin{equation}
\label{U-sm}
U\frac{\log (2m)}{n}\leq \sqrt{\frac{\log(2m)}{n m}},
\end{equation}
then the condition 
$\bar \eps \geq C_1U{\mathbb E}\|\Xi_{\eps}\|_{\infty}$ is satisfied 
and bound (\ref{intermed_12}) holds with probability at least $1-e^{-t}.$
Moreover, $\bar \eps \delta \lesssim_{D'} \delta \lesssim_{D'} U^2\frac{\bar t}{n},$
implying that the term $5\bar \eps \delta$ in (\ref{intermed_12}) can be dropped 
at a price of further increasing the value of constant $C.$

If (\ref{U-sm}) does not hold, we still have that 
$$
\|f_{\tilde \rho}-f_{\rho}\|_{L_2(\Pi)}^2 =\frac{\|\tilde \rho-\rho\|_2^2}{m^2}
\leq \frac{2}{m^2}\leq CU^2\frac{\bar t}{n}.
$$
Recalling that $\bar t\leq t+2\log (\log_2 (mn)+3)$ and $\log (m/\delta)\lesssim \log(mn),$
we deduce from (\ref{intermed_12}) that with some constant $C$ and with probability at least $1-e^{-t}$
\begin{eqnarray}
\label{intermed_13}
&
\nonumber
\|f_{\tilde \rho}-f_{\rho}\|_{L_2(\Pi)}^2
\leq 
\|f_{S'}-f_{\rho}\|_{L_2(\Pi)}^2
+ 
C\biggl[r m^2 \eps^2 \log^2 (mn)
\\
&
+  U^2\frac{r m \log(2m)}{n}+U^2\frac{t+\log (\log_2 (mn)+3)}{n}\biggr].
\end{eqnarray}
Note that, for $n\geq 2,$ 
\begin{equation}
\label{simplify1}
\log (\log_2(mn)+3) =
\log\Bigl(\log_2(4m)+\log_2(2n)\Bigr)
\leq \log \log_2(4m)+ \log\log_2(2n),
\end{equation}
since $\log_2(4m)+ \log_2(2n)\leq \log_2(4m)\log_2(2n).$
Since also, for $r\geq 1,$ 
\begin{equation}
\label{simplify2}
U^2\frac{t+\log \log_2 (4m)}{n}\lesssim U^2\frac{r m \log(2m)}{n},
\end{equation}
we can replace in bound (\ref{intermed_13}) the term $U^2\frac{t+\log (\log_2 (mn)+3)}{n}$
with the term $U^2\frac{t+\log \log_2 (2n)}{n}$ (increasing the value of the constant $C$ accordingly). This yields bound (\ref{sharp_oracle}) of the theorem. For $S'=\rho,$ 
it yields bound (\ref{L_2-error}), and, moreover, 
for $S'=\rho$ and $S=(1-\delta)\rho +\delta \frac{I_m}{m}$ with 
$\delta =\frac{1}{m^2 n^2},$
bound (\ref{intermed_12}) also implies that  
\begin{eqnarray}
\label{intermed_15}
&
\eps K(\tilde \rho;S)
\leq 
{\rm rank}(\rho) m^2 \eps^2 \log^2 (m/\delta)
+  2{\rm rank}(\rho) m^2 \bar \eps^2 
\\
&
\nonumber
+5\bar \eps\delta+ CU^2\frac{\bar t}{n}.
\end{eqnarray}
We will now take 
$$
\bar \eps:=D'\biggl[U\sqrt{\frac{\log(2m)}{n m}}\bigvee U^2\frac{\log (2m)}{n}\biggr]
$$
for a large enough constant $D'$ so that 
$
\bar \eps \geq C_1U{\mathbb E}\|\Xi_{\eps}\|_{\infty}.
$
Assume that 
$$
\eps:=\frac{1}{\log(mn)}\biggl[U\sqrt{\frac{\log(2m)}{n m}}\bigvee U^2\frac{\log (2m)}{n}\biggr].
$$
As before, the term $\bar \eps \delta$ in bound (\ref{intermed_15}) will be absorbed by the term
$CU^2\frac{\bar t}{n}$ with a larger value of $C$ and also  
$$
{\rm rank}(\rho) m^2 \eps^2 \log^2 (m/\delta)\asymp_{D'}{\rm rank}(\rho) m^2 \bar \eps^2 
\asymp_{D'} U^2\frac{{\rm rank}(\rho)m \log(2m)}{n}\biggl(1\bigvee U^2\frac{m\log(2m)}{n}\biggr). 
$$
As a result, taking into account (\ref{simplify1}), (\ref{simplify2}), bound  
(\ref{intermed_15}) can be rewritten 
as follows:
\begin{eqnarray}
\label{intermed_16}
&
\eps K(\tilde \rho;S)
\leq 
CU^2\biggl[\frac{{\rm rank}(\rho)m \log(2m)}{n}\biggl(1\bigvee U^2\frac{m\log(2m)}{n}\biggr)
\\
&
\nonumber
+\frac{t+\log\log_2(2n)}{n}\biggr].
\end{eqnarray}
Using the bound of Lemma \ref{lem-kull} along with the bound
$$
h(\delta)\leq \delta \log(e/\delta) = \frac{1}{m^2 n^2}\log (em^2n^2)
\lesssim U\sqrt{\frac{m}{n}}\frac{(t+\log \log_2 (2n))\log(mn)}{\sqrt{\log (2m)}},
$$ 
we easily get that (\ref{KL-error}) holds.  
\end{proof}

\subsection{Oracle Inequalities for Trace Regression with Gaussian Noise}
\label{vonneumann_2}

In this subsection, we establish oracle inequalities for the von Neumann entropy penalized least squares estimator 
$\tilde \rho^{\eps}$ in the case of trace regression model with Gaussian noise (Assumption \ref{Gaussian_noise}). 
Unlike in the case of Theorem \ref{th-KL-1}
of the previous section, our aim is not to obtain sharp oracle inequality, but rather to get a clean main term of the random error
bound part of the inequality, namely, the term $\sigma_{\xi}^2 \frac{{\rm rank}(S)m(t+\log (2m))}{n}$ in inequality (\ref{oracle_ineq}) below. Note that this term depends only on the variance of the noise $\sigma_{\xi}^2,$ but not on the constant $U$ from 
Assumption \ref{orthonorm} (the constant $U$ is involved only in the higher order $O(n^{-2})$ terms of the bound). 
Note also that there are no constraints on the variance $\sigma_{\xi}^2$ that could be arbitrarily small, or even equal 
to $0$ (in which case only higher order terms are present in the bound). This improvement comes at a price of having 
the leading constant $2$ in the oracle inequality and also of imposing assumption (\ref{ass_eps_1}) that requires the regularization 
parameter $\eps$ to be bounded away from $0$ (again, unlike Theorem \ref{th-KL-1}, where it could be arbitrarily small).
As in the previous section, we also obtain a bound on Kullback--Leibler divergence $K(\rho\|\tilde \rho^{\eps}).$

\begin{theorem}
\label{th-KL-2}
Let $t\geq 1.$ Suppose 
\begin{equation}
\label{ass_eps_1}
\eps \in \biggl[DU^2 \frac{t+\log^3 m \log^2 n}{n}, \frac{D_1\sigma_{\xi}}{\log (mn)}\sqrt{\frac{t+\log(2m)}{nm}}\bigvee 
DU^2 \frac{t+\log^3 m \log^2 n}{n}\biggr]
\end{equation}
with large enough constants $D,D_1>0.$ There exists a constant $C>0$ such that with probability at least $1-e^{-t}$
\begin{equation}
\label{oracle_ineq}
\begin{split}
&
\|f_{\tilde \rho^{\eps}}-f_{\rho}\|_{L_2(\Pi)}^2
\leq \inf_{S\in {\mathcal S}_m}
\biggl[2\|f_S-f_{\rho}\|_{L_2(\Pi)}^2 
+C\biggl(\sigma_{\xi}^2 \frac{{\rm rank}(S)m(t+\log (2m))}{n}
\\
&
+ \sigma_{\xi}^2 U^2 \frac{{\rm rank}(S)m^2 (t+\log (2m))^2 \log(2m)}{n^2}
+U^4 \frac{{\rm rank}(S)m^2 (t+\log^3 m \log^2 n)^2 \log^2(mn)}{n^2}
\biggr)\biggr].
\end{split}
\end{equation}
In particular, 
\begin{eqnarray}
\label{L_2-error''}
&
\|f_{\tilde \rho^{\eps}}-f_{\rho}\|_{L_2(\Pi)}^2
\leq C\biggl[\sigma_{\xi}^2 \frac{{\rm rank}(\rho) m(t+\log (2m))}{n}
\\
&
\nonumber
+ \sigma_{\xi}^2 U^2 \frac{{\rm rank}(\rho)m^2 (t+\log (2m))^2 \log(2m)}{n^2}
+U^4 \frac{{\rm rank}(\rho)m^2 (t+\log^3 m \log^2 n)^2 \log^2(mn)}{n^2}
\biggr].
\end{eqnarray}
Moreover, if 
$$
\eps := \frac{D_1\sigma_{\xi}}{\log (mn)}\sqrt{\frac{t+\log(2m)}{nm}}\bigvee 
DU^2 \frac{t+\log^3 m \log^2 n}{n}
$$
for large enough constants $D,D_1,$ then with some constant $C$ and with the same 
probability both (\ref{L_2-error''}) and the following bound hold:
\begin{eqnarray}
\label{KL-error''}
&
K(\rho\|\tilde \rho^{\eps})
\leq C\biggl[\sigma_{\xi} \frac{{\rm rank}(\rho) m^{3/2}(t+\log (2m))^{1/2}\log(mn)}{\sqrt{n}}
\\
&
\nonumber
+ \sigma_{\xi}^2 \frac{{\rm rank}(\rho)m^2 (t+\log (2m)) \log(2m)}{n}
+U^2 \frac{{\rm rank}(\rho)m^2 (t+\log^3 m \log^2 n) \log^2(mn)}{n}
\biggr].
\end{eqnarray}
\end{theorem}

\begin{proof} As in in the proof of Theorem \ref{th-KL-1}, we rely on Lemma \ref{lemma_intermed}, 
but we use a different approach to bounding  
the empirical process $(P-P_n)(\ell'\bullet f_{\tilde \rho})(f_{\tilde \rho}-f_S).$ 
The following identity follows from the definition of quadratic loss $\ell$
$$
(\ell^{\prime}\bullet f)(x,y)(f(x)-f_{S}(x))
=2(f(x)-f_{S}(x))^2 + 2(f_S(x)-y)(f(x)-f_{S}(x))
$$
and it implies that 
\begin{equation}
\label{nol'}
(P-P_n)(\ell^{\prime}\bullet f_{\tilde \rho})(f_{\tilde \rho}-f_S)=
-2(P_n-P)(f_{\tilde \rho}-f_S)^2 - 2\langle \Xi, \tilde \rho-S\rangle
\end{equation}
where 
$$
\Xi:= n^{-1}\sum_{j=1}^n (f_S(X_j)-Y_j)X_j-
{\mathbb E}(f_S(X)-Y)X.
$$
We will bound $(P_n-P)(f_{\tilde \rho}-f_S)^2$
in representation (\ref{nol'}) 
as follows:
\begin{equation}
\label{bd_square}
\Bigl|(P_n-P)(f_{\tilde \rho}-f_S)^2\Bigr|\leq 
\|\tilde \rho-S\|_1^2 \beta_n\biggl(\frac{\|f_{\tilde \rho}-f_S\|_{L_2(\Pi)}}{\|\tilde \rho-S\|_1}\biggr),
\end{equation}
where 
$$
\beta_n(\Delta):= \sup\biggl\{\Bigl|(P_n-P)(f_{A}^2)\Bigr|: A\in {\mathbb H}_m, \|A\|_1\leq 1,
\|f_A\|_{L_2(\Pi)}\leq \Delta\biggr\}.
$$

The next lemma provides a bound on $\beta_n(\Delta).$ Its proof is somewhat involved and 
it will not be given here. It is based on Rudelson's $L_{\infty}(P_n)$ generic chaining bound for empirical processes indexed by squares of functions 
and on the ideas of the paper by \cite{guedon2008} combined with Talagrand's concentration inequality (see also \citealt{aubrun2009}, \citealt{liu2011universal} and Theorem 3.16, Lemma 9.8 and Proposition 9.2 in \citealt{koltchinskii2011oracle} for similar arguments).

\begin{lemma}
\label{generic_ch}
Given $0<\delta^{-}<\delta^+$ and $t\geq 1,$ let 
$$
\bar t := t+\log \Bigl(\log_2 (\delta^+/\delta^-)+3\Bigr).
$$  
Then, with some constant $C$ and with probability at least $1-e^{-t},$ the following bound holds for all 
$\Delta\in [\delta^-,\delta^+]:$
\begin{equation}
\label{gen_bd}
\beta_n(\Delta)\leq 
C\biggl[\Delta U \frac{\log^{3/2}m \log n}{\sqrt{n}}+U^2\frac{\log^3 m \log^2 n}{n}
+\Delta U\sqrt{\frac{\bar t}{n}}+U^2\frac{\bar t}{n}\biggr].
\end{equation}
\end{lemma}

We will use Lemma \ref{generic_ch} to control $\beta_n(\Delta)$ for 
$\Delta:=\frac{\|f_{\tilde \rho}-f_S\|_{L_2(\Pi)}}{\|\tilde \rho-S\|_1}.$
Let $\delta^{+}:=\frac{1}{m}$ and $\delta^-:=\frac{1}{mn}.$ 
With this choice, $\bar t\leq t+\log (\log_2 n +3).$
Note that for $A=\frac{\tilde \rho-S}{\|\tilde \rho-S\|_1},$
$\|f_A\|_{L_2(\Pi)}=\frac{\|A\|_2}{m}\leq \frac{\|A\|_1}{m}= m^{-1}=\delta^+.$
If also $\|f_A\|_{L_2(\Pi)}\geq \delta^-,$ then we can substitute bound (\ref{gen_bd})
on $\beta_n(\Delta)$ into (\ref{bd_square}) that yields:
\begin{eqnarray}
\label{bd_square_1}
&
\nonumber
\Bigl|(P_n-P)(f_{\tilde \rho}-f_S)^2\Bigr|\leq 
C\biggl[\|f_{\tilde \rho}-f_S\|_{L_2(\Pi)}\|\tilde \rho-S\|_1 U \frac{\log^{3/2}m \log n}{\sqrt{n}}
\\
&
\nonumber
+
\|\tilde \rho-S\|_1^2U^2\frac{\log^3 m \log^2 n}{n}
+ \|f_{\tilde \rho}-f_S\|_{L_2(\Pi)}\|\tilde \rho-S\|_1U\sqrt{\frac{\bar t}{n}}
\\
&
\nonumber
+
\|\tilde \rho-S\|_1^2 U^2\frac{\bar t}{n}\biggr]
\\
&
\leq 
\frac{1}{32} \|f_{\tilde \rho}-f_S\|_{L_2(\Pi)}^2 + 
8(C^2+C/8)U^2 \frac{\log^{3}m \log^2 n}{n}\|\tilde \rho-S\|_1^2
\\
&
\nonumber
+
\frac{1}{32} \|f_{\tilde \rho}-f_S\|_{L_2(\Pi)}^2
+8(C^2+C/8)U^2\frac{\bar t}{n}\|\tilde \rho-S\|_1^2
\\
&
\nonumber
\leq 
\frac{1}{16} \|f_{\tilde \rho}-f_S\|_{L_2(\Pi)}^2
+ C'U^2 \frac{\log^{3}m \log^2 n+\bar t}{n}\|\tilde \rho-S\|_1^2,
\end{eqnarray}
where $C':=8(C^2+C/8).$
If, on the other hand, $\|f_A\|_{L_2(\Pi)}\leq \delta^-=\frac{1}{mn},$ then 
$\|f_{\tilde \rho}-f_S\|_{L_2(\Pi)}$ in the above bound can be 
replaced by $\frac{1}{mn}\|\tilde \rho-S\|_1$ and the proof that 
follows only simplifies since 
$$
\frac{1}{16} \|f_{\tilde \rho}-f_S\|_{L_2(\Pi)}^2\leq \frac{1}{16}\frac{1}{m^2n^2}\|\tilde \rho-S\|_1^2
\leq \frac{1}{16}U^2 \frac{\log^{3}m \log^2 n+\bar t}{n}\|\tilde \rho-S\|_1^2.
$$
 
Another term in the right hand side of representation (\ref{nol'}) to be 
controlled is $\langle \Xi, \tilde \rho-S\rangle.$ Note that $\Xi=\Xi_1+\Xi_2,$
where
$$
\Xi_1 := - n^{-1}\sum_{j=1}^n \xi_j X_j
$$
and 
$$
\Xi_2:=n^{-1}\sum_{j=1}^n (f_S(X_j)-f_{\rho}(X_j))X_j 
-{\mathbb  E}(f_S(X)-f_{\rho}(X))X.
$$

Recall that $S=(1-\delta)S'+\delta \frac{I_m}{m}$ with $S'\in {\mathcal S}_m,$ 
${\rm rank}(S')=r,$ ${\rm supp}(S')=L$ and $\delta=\frac{1}{m^2 n^2}.$

The term with $\Xi_1$ is controlled as follows:
\begin{eqnarray}
\label{Xi_1_bd}
&
\nonumber
\Bigl|\langle \Xi_1, \tilde \rho-S\rangle\Bigr| 
\\
&
\nonumber
\leq \Bigl|\langle {\mathcal P}_L(\Xi_1), \tilde \rho-S^{\prime}\rangle\Bigr|
+
\Bigl|\langle \Xi_1, {\mathcal P}_L^{\perp}(\tilde \rho-S^{\prime})\rangle\Bigr|
+
\Bigl|\langle {\mathcal P}_L^{\perp}(\Xi_1), S^{\prime}-S\rangle\Bigr|
\\
&
\nonumber
\leq \|{\mathcal P}_L(\Xi_1)\|_2 \|\tilde \rho-S^{\prime}\|_2 + 
\|\Xi_1\|_{\infty}\|{\mathcal P}_L^{\perp}(\tilde \rho)\|_1
+\Bigl\|{\mathcal P}_L^{\perp}(\Xi_1)\Bigr\|_{\infty}\|S^{\prime}-S\|_1
\\
&
\leq 2\sqrt{2r} m \|\Xi_1\|_{\infty} \|f_{\tilde \rho}-f_S\|_{L_2(\Pi)}
+ \|\Xi_1\|_{\infty}\|{\mathcal P}_L^{\perp}(\tilde \rho)\|_1
+4\delta \|\Xi_1\|_{\infty}
\\
&
\nonumber
\leq 
32 rm^2\|\Xi_1\|_{\infty}^2 + \frac{1}{16} \|f_{\tilde \rho}-f_S\|_{L_2(\Pi)}^2
\\
&
\nonumber
+\|\Xi_1\|_{\infty}\|{\mathcal P}_L^{\perp}(\tilde \rho)\|_1
+4\delta \|\Xi_1\|_{\infty}.
\end{eqnarray}
We also have 
\begin{eqnarray}
\label{XI_2_bd}
&
\nonumber
\Bigl|\langle \Xi_2, \tilde \rho-S\rangle\Bigr|\leq \|\Xi_2\|_{\infty}\|\tilde \rho-S\|_1
\leq \|\Xi_2\|_{\infty}\|\tilde \rho-S'\|_{1}+ \|\Xi_2\|_{\infty}\|S'-S\|_1
\\
&
\leq \|\Xi_2\|_{\infty}\|\tilde \rho-S'\|_{1}+ 2\delta \|\Xi_2\|_{\infty}.
\end{eqnarray} 
Thus,
\begin{eqnarray}
\label{Xi_bd}
&
\nonumber
\Bigl|\langle \Xi, \tilde \rho-S\rangle\Bigr|\leq 
32 rm^2\|\Xi_1\|_{\infty}^2 + \frac{1}{16} \|f_{\tilde \rho}-f_S\|_{L_2(\Pi)}^2
\\
&
+\|\Xi_1\|_{\infty}\|{\mathcal P}_L^{\perp}(\tilde \rho)\|_1
+4\delta \|\Xi_1\|_{\infty} + 
\|\Xi_2\|_{\infty}\|\tilde \rho-S'\|_{1}+ 2\delta \|\Xi_2\|_{\infty}.
\end{eqnarray}
It follows from (\ref{nol'}), (\ref{bd_square_1}) and (\ref{Xi_bd}) that
with some constant $C'$
\begin{eqnarray} 
&
\nonumber
(P-P_n)(\ell'\bullet f_{\tilde \rho})(f_{\tilde \rho}-f_S)\leq 
\\
&
\frac{1}{4} \|f_{\tilde \rho}-f_S\|_{L_2(\Pi)}^2
+ C'U^2 \frac{\log^{3}m \log^2 n+\bar t}{n}\|\tilde \rho-S\|_1^2
\\
&
\nonumber
+ 64rm^2\|\Xi_1\|_{\infty}^2 
+2\|\Xi_1\|_{\infty}\|{\mathcal P}_L^{\perp}(\tilde \rho)\|_1
+8\delta \|\Xi_1\|_{\infty} 
\\
&
\nonumber
+ 
2\|\Xi_2\|_{\infty}\|\tilde \rho-S'\|_{1}+ 4\delta\|\Xi_2\|_{\infty}.
\end{eqnarray}
This bound will be substituted in (\ref{intermed_2}).
Note that, if assumption (\ref{ass_eps_1}) on $\eps$ holds with a sufficiently large 
constant $D,$ then we have 
$$
\eps \geq 8C'U^2 \frac{\log^{3}m \log^2 n+\bar t}{n}
$$
(this follows from the fact that $\bar t\leq t+\log(\log_2 n+3)\leq t+ c\log^3 m \log^2 n$ for some constant $c>0$). 
Assume also that $\bar \eps \geq 4\|\Xi_1\|_{\infty}$ and recall that 
$K(\tilde \rho;S)\geq \frac{1}{4}\|\tilde \rho-S\|_1^2$
(see inequality \ref{khtracelem}). 
Taking all this into account, (\ref{intermed_2}) implies that 
\begin{eqnarray}
\label{intermed_100}
&
\nonumber
\|f_{\tilde \rho}-f_{\rho}\|_{L_2(\Pi)}^2+\frac{1}{4}\|f_{\tilde \rho}-f_S\|_{L_2(\Pi)}^2
+ \frac{\eps}{2} K(\tilde \rho;S)+\frac{\bar \eps}{2}\|{\mathcal P}_L^{\perp}\tilde \rho\|_1
\\
&
\leq 
\|f_S-f_{\rho}\|_{L_2(\Pi)}^2+ r m^2 \eps^2 \log^2 (m/\delta)
+ 5 r m^2 \bar \eps^2 
+6\bar \eps\delta 
\\
&
\nonumber
+2\|\Xi_2\|_{\infty}\|\tilde \rho-S'\|_{1}+ 4\|\Xi_2\|_{\infty}\delta.
\end{eqnarray}

It remains to control $\|\Xi_1\|_{\infty}$ and $\|\Xi_2\|_{\infty}.$
To this end, we use matrix versions of Bernstein inequality.
To bound $\|\Xi_2\|_{\infty},$ we use its standard version 
which yields that with probability at least $1-e^{-t}$
\begin{eqnarray*}
&
\|\Xi_2\|_{\infty}\leq 
2\biggl[\Bigl\|{\mathbb E}(f_S(X)-f_{\rho}(X))^2 X^2\Bigr\|_{\infty}^{1/2}
\sqrt{\frac{t+\log(2m)}{n}}
\\
&
\bigvee \Bigl\|(f_S(X)-f_{\rho}(X))\|X\|_{\infty}\Bigr\|_{L_{\infty}}\frac{t+\log (2m)}{n}\biggr],
\end{eqnarray*}
where $\|\cdot\|_{L_{\infty}}$ denotes the essential supremum norm in the space 
of random variables. 
Since 
$$
\Bigl\|{\mathbb E}(f_S(X)-f_{\rho}(X))^2 X^2\Bigr\|_{\infty}
\leq U^2 \|f_S-f_{\rho}\|_{L_2(\Pi)}^2
$$
and 
$$
\Bigl\|(f_S(X)-f_{\rho}(X))\|X\|_{\infty}\Bigr\|_{L_{\infty}}
\leq 2U^2,
$$
we get 
\begin{eqnarray}
\label{Xi_2_bou}
&
\|\Xi_2\|_{\infty}
\leq 
4\biggl[\|f_S-f_{\rho}\|_{L_2(\Pi)}U\sqrt{\frac{t+\log(2m)}{n}}+
U^2\frac{t+\log (2m)}{n}\biggr].
\end{eqnarray}
This implies that 
\begin{eqnarray}
\label{Xi_2_boun}
&
2\|\Xi_2\|_{\infty}\|\tilde \rho-S'\|_1
\leq \|f_S-f_{\rho}\|_{L_2(\Pi)}^2 + 
16U^2\frac{t+\log (2m)}{n}\|\tilde \rho-S'\|_1^2
\\
&
\nonumber
+8U^2\frac{t+\log (2m)}{n}\|\tilde \rho-S'\|_1.
\end{eqnarray}
Note that 
\begin{eqnarray}
\label{Xi_2_1}
&
\nonumber
16U^2\frac{t+\log (2m)}{n}\|\tilde \rho-S'\|_1^2
\\
&
\leq 16U^2\frac{t+\log (2m)}{n}\|\tilde \rho-S\|_1^2
+16U^2\frac{t+\log (2m)}{n}(4\delta+\delta^2)
\end{eqnarray}
and 
\begin{eqnarray}
\label{Xi_2_2}
&
\nonumber
8U^2\frac{t+\log (2m)}{n}\|\tilde \rho-S'\|_1
\\
&
\leq 8U^2\frac{t+\log (2m)}{n}\|{\mathcal P}_L^{\perp}\tilde \rho\|_1 
+ 8U^2\frac{t+\log (2m)}{n}\|{\mathcal P}_L(\tilde \rho-S')\|_1
\\
&
\nonumber
\leq 
8U^2\frac{t+\log (2m)}{n}\|{\mathcal P}_L^{\perp}\tilde \rho\|_1 
+ 8U^2\frac{t+\log (2m)}{n}\|{\mathcal P}_L(\tilde \rho-S)\|_1
+16U^2\frac{t+\log (2m)}{n}\delta.
\end{eqnarray}
Since, for some constant $C''>0,$ 
\begin{eqnarray*}
&
\nonumber
8U^2\frac{t+\log (2m)}{n}\|{\mathcal P}_L(\tilde \rho-S)\|_1
\leq 
8\sqrt{2}U^2\frac{t+\log (2m)}{n}\sqrt{r}\|{\mathcal P}_L(\tilde \rho-S)\|_2
\\
&
\nonumber
\leq 
8\sqrt{2}U^2\frac{t+\log (2m)}{n}\sqrt{r}m
\|f_{\tilde \rho}-f_S\|_{L_2(\Pi)}
\nonumber
\leq 
\frac{1}{4}\|f_{\tilde \rho}-f_S\|_{L_2(\Pi)}^2
+C'' U^4 \frac{rm^2(t+\log (2m))^2}{n^2},
\end{eqnarray*}
it follows from (\ref{Xi_2_boun}), (\ref{Xi_2_1}) and (\ref{Xi_2_2}) that 
\begin{eqnarray}
\label{Xi_2_boun'}
&
\nonumber
2\|\Xi_2\|_{\infty}\|\tilde \rho-S'\|_1
\leq \|f_S-f_{\rho}\|_{L_2(\Pi)}^2 +
\\
& 
+16U^2\frac{t+\log (2m)}{n}\|\tilde \rho-S\|_1^2
+16U^2\frac{t+\log (2m)}{n}(4\delta+\delta^2)
\\
&
\nonumber
+8U^2\frac{t+\log (2m)}{n}\|{\mathcal P}_L^{\perp}\tilde \rho\|_1 
+16U^2\frac{t+\log (2m)}{n}\delta
\\
&
\nonumber
+\frac{1}{4}\|f_{\tilde \rho}-f_S\|_{L_2(\Pi)}^2
+C'' U^4 \frac{rm^2(t+\log (2m))^2}{n^2}.
\end{eqnarray}
Note that (\ref{Xi_2_bou}) also implies that 
\begin{eqnarray}
\label{Xi_2_bou'}
&
\|\Xi_2\|_{\infty}
\leq 
4\biggl[\frac{2U}{m}\sqrt{\frac{t+\log(2m)}{n}}+
U^2\frac{t+\log (2m)}{n}\biggr]
\end{eqnarray}
(since $\|f_S-f_{\rho}\|_{L_2(\Pi)}\leq m^{-1}\|S-\rho\|_2\leq 2m^{-1}$). 
Let us substitute (\ref{Xi_2_boun'}) and (\ref{Xi_2_bou'}) in the last line of (\ref{intermed_100}).
Assume that 
$$
\bar \eps \geq 16 U^2\frac{t+\log (2m)}{n}
$$
and that constant $D$ in assumption (\ref{ass_eps_1}) is large enough so that   
$$
16U^2\frac{t+\log (2m)}{n}\|\tilde \rho-S\|_1^2 \leq \frac{\eps}{4}K(\tilde \rho, S)
$$
(recall inequality \ref{khtracelem}).
It easily follows that with some constants $C_1, C_2,$
\begin{eqnarray}
\label{intermed_101}
&
\nonumber
\|f_{\tilde \rho}-f_{\rho}\|_{L_2(\Pi)}^2
+ \frac{\eps}{4} K(\tilde \rho;S)
\\
&
\leq 
2\|f_S-f_{\rho}\|_{L_2(\Pi)}^2+ C_1 r m^2 \eps^2 \log^2 (m/\delta)
+ 5 r m^2 \bar \eps^2 
\\
&
\nonumber 
+C_2\bar \eps\delta 
+
32\frac{U}{m}\sqrt{\frac{t+\log(2m)}{n}}\delta
\end{eqnarray}
(note that the term $C'' U^4 \frac{rm^2(t+\log (2m))^2}{n^2}$ of bound (\ref{Xi_2_boun'}) is ``absorbed" by the 
term $C_1 r m^2 \eps^2 \log^2 (m/\delta)$ of bound (\ref{intermed_101}) provided that 
constant $C_1$ is large enough). 
Since 
$$
\delta=\frac{1}{m^2n^2}\leq U^2\frac{t+\log(2m)}{n}
\leq \bar \eps 
$$
(recall that $U^2\geq m^{-1}$), 
we have $\bar \eps\delta \leq \bar \eps^2.$ 
Also, since $U\geq m^{-1/2},$
$$
\frac{U}{m}\sqrt{\frac{t+\log(2m)}{n}}\delta= U\sqrt{\frac{t+\log(2m)}{n}}\frac{1}{m^3 n^2}
\leq U^4 \biggl(\frac{t+\log (2m)}{n}\biggr)^2 \leq \bar \eps^2.
$$
Therefore, (\ref{intermed_101}) implies that with some constant $C$
\begin{eqnarray}
\label{intermed_102}
&
\nonumber
\|f_{\tilde \rho}-f_{\rho}\|_{L_2(\Pi)}^2
+ \frac{\eps}{4} K(\tilde \rho;S)
\\
&
\leq 
2\|f_S-f_{\rho}\|_{L_2(\Pi)}^2+ C\Bigl(r m^2 \eps^2 \log^2 (m/\delta)
+ r m^2 \bar \eps^2\Bigr). 
\end{eqnarray}

To bound $\|\Xi_1\|_{\infty},$ we use a version 
of matrix Bernstein type inequality due to \cite{koltchinskii2011oracle} (see bound (2.7) of Theorem 2.7).
Its version for $\alpha=2$ (with $U^{(\alpha)}\asymp U\sigma_{\xi}$) 
implies that for some constant $K>0$ with probability 
at least $1-e^{-t}$
\begin{equation}
\|\Xi_1\|_{\infty} \leq K\biggl[\sigma_{\xi}\sqrt{\frac{t+\log(2m)}{nm}}\bigvee 
\sigma_{\xi}U\frac{(t+\log(2m))\log^{1/2}(2Um^{1/2})}{n}\biggr].
\end{equation}
We choose 
$$
\bar \eps:= D_2\biggl[\sigma_{\xi}\sqrt{\frac{t+\log(2m)}{nm}}\bigvee (\sigma_{\xi}\vee U)U
\frac{(t+\log (2m))\log^{1/2}(2m)}{n}\biggr]
$$
with a sufficiently large constant $D_2$ to satisfy the condition $\|\Xi_1\|_{\infty}\leq 4\bar \eps$
with probability at least $1-e^{-t}$ (the rest of the assumptions we made on $\bar \eps$
are also satisfied with this choice).

Bound (\ref{intermed_102}) then implies that with some constant $C$ and with probability 
at least $1-3e^{-t}$ the following inequality holds:
\begin{equation}
\label{intermed_103}
\begin{split}
&
\|f_{\tilde \rho^{\eps}}-f_{\rho}\|_{L_2(\Pi)}^2
\leq 
2\|f_S-f_{\rho}\|_{L_2(\Pi)}^2 
\\
&
+C\biggl[\sigma_{\xi}^2 \frac{rm(t+\log (2m))}{n}
+ \sigma_{\xi}^2 U^2 \frac{rm^2 (t+\log (2m))^2 \log(2m)}{n^2}
\\
&
+U^4 \frac{rm^2 (t+\log^3 m \log^2 n)^2 \log^2(mn)}{n^2}
\biggr].
\end{split}
\end{equation}
Using bound (\ref{S'S}) to replace $S$ in $\|f_S-f_{\rho}\|_{L_2(\Pi)}^2$ with $S'$ and
adjusting the value of constant $C$ to rewrite the probability
bound as $1-e^{-t},$ it is easy to complete the proof of (\ref{oracle_ineq}). 
If $S'=\rho,$ this also yields bound (\ref{L_2-error''}). 
Moreover, with a larger value of regularization parameter 
$$
\eps := \frac{D_1\sigma_{\xi}}{\log (mn)}\sqrt{\frac{t+\log(2m)}{nm}}\bigvee 
DU^2 \frac{t+\log^3 m \log^2 n}{n},
$$
bound (\ref{intermed_102}) and Lemma \ref{lem-kull} easily imply 
bound (\ref{KL-error''}).
\end{proof}

\subsection{Optimality Properties of von Neumann Entropy Penalized Estimator $\tilde{\rho}^{\epsilon}$}
\label{tildesec}

We start with upper bounds on the error of estimator $\tilde{\rho}^{\epsilon}$ (von Neumann entropy penalized least squares estimator 
defined by (\ref{trwvne})) in Hellinger, Kullback-Leibler and Schatten $q$-norm distances for $q\in [1,2]$ for the trace regression model with Gaussian noise (Assumption \ref{Gaussian_noise}). To avoid the impact 
of ``second order terms" on the upper bounds, we will make the following simplifying assumptions:
\begin{equation}
\label{simplifying_assumption}
U\sqrt{\frac{m}{n}}\log m\lesssim 1\ \ {\rm and}\ \ U^2\sqrt{\frac{m}{n}}\log^{5/2}m \log^2 n \log(mn)\lesssim \sigma_{\xi}.
\end{equation} 
Recall that, for the Pauli basis, $U=m^{-1/2},$ so, the above assumptions hold if $n\gtrsim \log^2 m$ and $\sigma_{\xi}$
is larger than $\frac{1}{\sqrt{mn}}$ (times a logarithmic factor).  
We will choose regularization parameter $\eps$ as follows:
\begin{equation}
\label{ass_eps_1'''}
\eps :=\frac{D_1\sigma_{\xi}}{\log (mn)}\sqrt{\frac{\log(2m)}{nm}}
\end{equation}
with a sufficiently large constant $D_1>0.$
The next result shows that minimax rates of Theorem \ref{minmaxthm1} 
are attained up to logarithmic factors for the estimator $\tilde \rho^{\eps}.$

\begin{theorem}
\label{upper_bd_Gauss}
There exists a constant $C>0$ such that the following bounds hold for all $r=1,\dots, m,$ for all $\rho\in {\mathcal S}_{r,m}$
and for all $q\in [1,2]$ with probability at least $1-m^{-2}:$
\begin{equation}
\label{upperthm1boundq}
\|\tilde{\rho}^{\eps}-\rho\|_q\leq C\biggl(\frac{\sigma_{\xi}m^{\frac{3}{2}}r^{1/q}}{\sqrt{n}}\sqrt{\log m}\log^{(2-q)/q}(mn)
\bigwedge 
\biggl(\frac{\sigma_{\xi}m^{3/2}}{\sqrt{n}}\biggr)^{1-\frac{1}{q}}(\log m)^{\frac{1}{2}-\frac{1}{2q}}\biggr)\bigwedge 2,
\end{equation} 
\begin{equation}
\label{upperthm1boundH}
H^2(\tilde{\rho}^{\eps},\rho)\leq C\frac{\sigma_{\xi}m^{\frac{3}{2}}r}{\sqrt{n}}\sqrt{\log m}\log(mn)\bigwedge 2 
\end{equation}
and
\begin{equation}
\label{upperthm1boundK}
K(\rho\|\tilde{\rho}^{\eps})\leq C\frac{\sigma_{\xi}m^{\frac{3}{2}}r}{\sqrt{n}}\sqrt{\log m}\log(mn).
\end{equation}
\end{theorem}

\begin{proof}
We will need the following simple lemma.

\begin{lemma}
For all $\rho\in {\mathcal S}_m$ and all $l=1,\dots, m,$ there exists $\rho'\in {\mathcal S}_{l,m}$
such that 
$$\|\rho-\rho'\|_2^2\leq \frac{1}{l}.$$
\end{lemma}

\begin{proof}
Suppose that $\rho=\sum_{j=1}^m{\lambda_j}P_j,$ where $\lambda_j$ are the eigenvalues of $\rho$ repeated 
with their multiplicities and $P_j$ are orthogonal one-dimensional projectors. Note that $\{\lambda_j: j=1,\dots, m\}$
is a probability distribution on the set $\{1,\dots, m\}.$ Let $\nu $ be a random variable sampled from this distribution 
and $\nu_1,\dots, \nu_l$ be its i.i.d. copies. Then ${\mathbb E}P_{\nu}=\rho$ and 
$$
{\mathbb E}\biggl\|l^{-1}\sum_{j=1}^l P_{\nu_j}- \rho\biggr\|_2^2=\frac{{\mathbb E}\|P_{\nu}-\rho\|_2^2}{l}
=\frac{{\mathbb E}\|P_{\nu}\|_2^2-\|\rho\|_2^2}{l}=\frac{1-\|\rho\|_2^2}{l}\leq\frac{1}{l}.
$$
Therefore, there exists a realization $\nu_1=k_1,\dots, \nu_l=k_l$ of r.v. $\nu_1,\dots, \nu_l$ such that 
$$
\biggl\|l^{-1}\sum_{j=1}^l P_{k_j}- \rho\biggr\|_2^2\leq \frac{1}{l}.
$$
Denote $\rho':=l^{-1}\sum_{j=1}^l P_{k_j}.$ Then, $\rho'\in {\mathcal S}_{l,m}$ and $\|\rho-\rho'\|_2^2\leq \frac{1}{l}.$
\end{proof}

First, we will prove bound (\ref{upperthm1boundq}) for $q=2.$ To this end, we use oracle inequality (\ref{oracle_ineq})
with $t=2\log m + \log 2$ and with oracle $S=\rho'\in {\mathcal S}_{l,m}$ such that $\|\rho-\rho'\|_2^2\leq \frac{1}{l}.$
Under simplifying assumptions (\ref{simplifying_assumption}) it yields that with probability at least $1-\frac{1}{2}m^{-2}$ 
$$
\|\tilde \rho^{\eps}-\rho\|_2^2 = m^2 \|f_{\tilde \rho^{\eps}}-f_{\rho}\|_{L_2(\Pi)}^2 
\lesssim \biggl[\frac{1}{l}+ \tau^2 l \log m\biggr],
$$
where $\tau := \frac{\sigma_{\xi}m^{3/2}}{\sqrt{n}}.$ On the other hand, using the same inequality with $S=\rho\in {\mathcal S}_{r,m}$
yields the bound 
$$
\|\tilde \rho^{\eps}-\rho\|_2^2
\lesssim \tau^2 r \log m
$$
that also holds with probability at least $1-\frac{1}{2}m^{-2}.$ Therefore, with probability at least $1-m^{-2}$
\begin{equation}
\label{ltau}
\|\tilde \rho^{\eps}-\rho\|_2^2 \lesssim \Bigl(\frac{1}{l}+ \tau^2 l \log m\Bigr) \bigwedge \tau^2 r \log m.
\end{equation}
Let $\bar l= \frac{1}{\tau \sqrt{\log m}}.$ If $\bar l\in [1,m],$ set $l:=[\bar l].$ Otherwise,
if $\bar l>m,$ set $l:=m$ and, if $\bar l<1,$ set $l:=1.$ An easy computation shows 
that with such a choice of $l$ bound (\ref{ltau}) implies (\ref{upperthm1boundq}) for $q=2.$

Next we use bound (\ref{KL-error''}) that, for $t=2\log m,$ implies under assumptions (\ref{simplifying_assumption})
that with some constant $C$ and with probability at least $1-m^{-2}$
\begin{eqnarray}
\label{KL-error+}
&
K(\rho\|\tilde \rho^{\eps})
\leq C\sigma_{\xi} \frac{r m^{3/2}\sqrt{\log m}\log(mn)}{\sqrt{n}},
\end{eqnarray}
which is bound (\ref{upperthm1boundK}). Bound (\ref{upperthm1boundH}) also holds 
in view of inequality (\ref{khtracelem}). 

Now, we prove bound (\ref{upperthm1boundq}) for $q=1$ (the bound for $q\in [1,2]$ will then follow by interpolation). 
To this end, we will use the following lemma (see Proposition 1 in \citealt{koltchinskii2011neumann}) that shows that 
if two density matrices are close in Hellinger distance and one of them is ``concentrated around a subspace" $L,$
then another one is also ``concentrated around" $L.$

\begin{lemma}
For any $L\subset {\mathbb C}^m$ and all $S_1,S_2\in {\mathcal S}_m,$
$$
\|{\mathcal P}_L^{\perp} S_1\|_1 \leq 2\|{\mathcal P}_L^{\perp} S_2\|_1 + 2H^2(S_1,S_2).
$$
\end{lemma}
We apply this lemma to $S_1=\tilde \rho^{\eps},$ $S_2=\rho$ and $L={\rm supp}(\rho)$
so that ${\mathcal P}_L^{\perp} \rho=0.$
It yields that 
$$
\|{\mathcal P}_L^{\perp} \tilde \rho^{\eps}\|_1 \leq 2H^2(\tilde \rho^{\eps},\rho).
$$
Therefore,
\begin{equation}
\label{1_2_H}
\|\tilde \rho^{\eps}-\rho\|_1 \leq 
\|{\mathcal P}_L(\tilde \rho^{\eps}-\rho)\|_1
+
\|{\mathcal P}_L^{\perp}(\tilde \rho^{\eps}-\rho)\|_1
\leq \sqrt{2r}\|\tilde \rho^{\eps}-\rho\|_2 + \|{\mathcal P}_L^{\perp} \tilde \rho^{\eps}\|_1
\leq \sqrt{2r}\|\tilde \rho^{\eps}-\rho\|_2+ 2H^2(\tilde \rho^{\eps},\rho).
\end{equation}
Using bounds (\ref{upperthm1boundq}) for $q=2$ and (\ref{upperthm1boundH}), 
we get from (\ref{1_2_H}) that 
\begin{equation}
\label{upperthm1bound1}
\|\tilde{\rho}^{\eps}-\rho\|_1\leq C\frac{\sigma_{\xi}m^{\frac{3}{2}}r}{\sqrt{n}}\sqrt{\log m}\log (mn)
\bigwedge 2,
\end{equation} 
which is equivalent to (\ref{upperthm1boundq}) for $q=1.$ Note that by choosing $t= 2\log m +\log 2+2$ (which might 
have an impact only on the constant), we could make probability bounds in (\ref{upperthm1boundq}) for $q=2$
and (\ref{upperthm1boundH}) to be at least $1-\frac{1}{2}m^{-2}$ implying that (\ref{upperthm1bound1})
holds with probability at least $1-m^{-2},$ as it is claimed in the theorem. 

To complete the proof, it is enough to use the interpolation inequality of Lemma \ref{interlem}. It follows that, for $q\in (1,2),$ 
$$
\|\tilde \rho^{\eps}-\rho\|_q \leq \|\tilde \rho^{\eps}-\rho\|_1^{\frac{2}{q}-1}\|\tilde \rho^{\eps}-\rho\|_2^{2-\frac{2}{q}}.
$$
Substituting bound (\ref{upperthm1boundq}) for $q=1$ and $q=2$ into the last inequality yields the result for an arbitrary 
$q\in (1,2).$
\end{proof}

Similarly, in the case of trace regression with bounded response (see Assumption \ref{bounded_response}), 
minimax rates of Theorem \ref{minmaxthm3} are also attained for the estimator $\tilde \rho^{\eps}$ (up to log factors).
In this case, assume that Assumption \ref{bounded_response} holds with $\bar U=U$ and, in addition, let us make the 
following simplifying assumptions:
\begin{equation}
\label{simplifying_assumption''}
U\sqrt{\frac{m\log m}{n}}\lesssim 1\ \ {\rm and}\ \ \log\log_2 n \lesssim m\log m.
\end{equation}
For the Pauli basis ($U=m^{-1/2}$), the first assumption holds if $n\gtrsim \log m.$
The second assumption does hold unless $n$ is extremely large ($n\sim 2^{\exp\{m\log m\}}$).  
Under these assumptions, we will use the following value of regularization parameter $\eps:$
$$
\eps:=\frac{U}{\log(mn)}\sqrt{\frac{\log(2m)}{n m}}.
$$
The following version of Theorem \ref{upper_bd_Gauss} holds in the bounded regression case
(with a similar proof).

\begin{theorem}
\label{upper_bd_bound}
There exists a constant $C>0$ such that the following bounds hold for all $r=1,\dots, m,$ for all $\rho\in {\mathcal S}_{r,m}$
and for all $q\in [1,2]$ with probability at least $1-m^{-2}:$
\begin{equation}
\label{upperthm2boundq}
\|\tilde{\rho}^{\eps}-\rho\|_q\leq C\biggl(\frac{U m^{\frac{3}{2}}r^{1/q}}{\sqrt{n}}\sqrt{\log m}\log^{(2-q)/q}(mn)
\bigwedge 
\biggl(\frac{U m^{3/2}}{\sqrt{n}}\biggr)^{1-\frac{1}{q}}(\log m)^{\frac{1}{2}-\frac{1}{2q}}\biggr)\bigwedge 2,
\end{equation} 
\begin{equation}
\label{upperthm2boundH}
H^2(\tilde{\rho}^{\eps},\rho)\leq C\frac{U m^{\frac{3}{2}}r}{\sqrt{n}}\sqrt{\log m}\log(mn)\bigwedge 2 
\end{equation}
and
\begin{equation}
\label{upperthm2boundK}
K(\rho\|\tilde{\rho}^{\eps})\leq C\frac{U m^{\frac{3}{2}}r}{\sqrt{n}}\sqrt{\log m}\log(mn).
\end{equation}
\end{theorem}

\begin{remark}
In the case of Pauli basis, the minimax optimal rates (up to constants and logarithmic factors) are: 
$\frac{mr^{1/q}}{\sqrt{n}}\wedge (\frac{m}{\sqrt{n}})^{1-\frac{1}{q}}\wedge 2$ for Schatten $q$-norm distances  
for $q\in [1,2];$ $\frac{m r}{\sqrt{n}}$ for nuclear norm, squared Hellinger and Kullback-Leibler distances (provided 
the $mr\lesssim \sqrt{n}$).
\end{remark}

\bibliographystyle{abbrv}
\bibliography{refer}

\end{document}